\documentclass[acmlarge,screen]{acmart}
\usepackage{bm}
\usepackage{amsmath}
\usepackage{algorithm}
\usepackage{algorithmic}
\usepackage[flushleft]{threeparttable}
\usepackage{multirow}
\usepackage{tabularx}
\usepackage{multicol}

\AtBeginDocument{%
  \providecommand\BibTeX{{%
    \normalfont B\kern-0.5em{\scshape i\kern-0.25em b}\kern-0.8em\TeX}}}

\setcopyright{acmcopyright}
\copyrightyear{2023}
\acmYear{2023}
\acmDOI{XXXXXXX.XXXXXXX}

\acmConference[]{XX}{2024}{Somewhere}
\acmPrice{15.00}
\acmISBN{978-1-4503-XXXX-X/18/06}

\begin{document}

\title{HyperSNN: A new efficient and robust deep learning model for resource constrained control applications}


\author{Zhanglu Yan}
\affiliation{%
  \institution{National University of Singapore}
  \streetaddress{P.O. Box 1212}
  \city{Singapore}
  \country{Singapore}
  \postcode{43017-6221}
}

\author{Shida Wang}
\affiliation{%
  \institution{National University of Singapore}
  \streetaddress{1 Th{\o}rv{\"a}ld Circle}
  \city{Singapore}
  \country{Singapore}}

\author{Kaiwen Tang}
\affiliation{%
  \institution{National University of Singapore}
  \streetaddress{1 Th{\o}rv{\"a}ld Circle}
  \city{Singapore}
  \country{Singapore}}

\author{Weng-Fai Wong}
\affiliation{%
  \institution{National University of Singapore}
  \streetaddress{1 Th{\o}rv{\"a}ld Circle}
  \city{Singapore}
  \country{Singapore}}

\renewcommand{\shortauthors}{Yan, et al.}

\begin{abstract}

In light of the increasing adoption of edge computing in areas such as intelligent furniture, robotics, and smart homes, this paper introduces {\em HyperSNN}, an innovative method for control tasks that uses spiking neural networks (SNNs) in combination with hyperdimensional computing. HyperSNN substitutes expensive 32-bit floating point multiplications with 8-bit integer additions, resulting in reduced energy consumption while enhancing robustness and potentially improving accuracy. Our model was tested on AI Gym benchmarks, including Cartpole, Acrobot, MountainCar, and Lunar Lander. HyperSNN achieves control accuracies that are on par with conventional machine learning methods but with only 1.36\% to 9.96\% of the energy expenditure. Furthermore, our experiments showed increased robustness when using HyperSNN. We believe that HyperSNN is especially suitable for interactive, mobile, and wearable devices, promoting energy-efficient and robust system design. Furthermore, it paves the way for the practical implementation of complex algorithms like model predictive control (MPC) in real-world industrial scenarios.

\end{abstract}



\received{20 February 2007}
\received[revised]{12 March 2009}
\received[accepted]{5 June 2009}

\maketitle

\section{Introduction}

In environments with limited computational capabilities, like intelligent furniture, robotics, smart homes, or wearables, it's common to offload raw data to cloud infrastructure, including fog devices or remote servers, for processing~\cite{rashid2020hear,soliman2013smart,article_li,majumder2017smart}. This offloading, however, can be energy-intensive, impacting device battery life~\cite{rashid2022template,mamaghanian2011compressed} and introducing latency, potentially affecting real-time processing~\cite{aburukba2020scheduling}. To mitigate these challenges, edge computing, where data is processed directly on the originating device with only results sent to the cloud, has become more prominent~\cite{chen2021recent,shi2016edge,li2018learning}. While this minimizes energy and latency concerns, it highlights the importance of energy-efficient computing for edge applications~\cite{jiang2020energy}. For tasks such as automated control in robotics and intelligent furniture, the constrained computational resources necessitate highly efficient methods~\cite{hu2019irobot,rashid2022template}. Furthermore, certain devices might have sensors with limited sensitivity due to hardware constraints~\cite{mehta2012mobile}. Given the diverse noise in real-world scenarios, ensuring model robustness is essential.

For control challenges, most current solutions focus on performance, frequently using intricate machine learning or deep learning models like MLP, VGG, ResNet, and MobileNet~\cite{de2018end,lockwood2020reinforcement,wang2020novel}. Yet, these often neglect energy efficiency and robustness. In contrast, {\em spiking neural networks} (SNNs) use binary inputs and outputs, substituting power-hungry multiplications with simpler additions, which are more energy efficient and silicon-resource friendly~\cite{ghosh2009spiking}. The inherent quantization in SNNs bolsters their noise resilience, making them apt for data from low-sensitivity sensors~\cite{zenke2021remarkable,sharmin2019comprehensive}. However, the energy-consuming softmax function in their final classification layer can hinder efficiency. Our research recommends pairing SNNs with {\em hyperdimensional computing} (HDC)~\cite{ge2020classification}, which uses xor operations on binary hypervectors for power savings. This approach further diminishes energy use and bolsters robustness, providing an optimal solution for constrained environments needing efficient and sturdy computation.

\begin{figure}[ht]
    \centering
    \includegraphics[width = 0.75\linewidth]{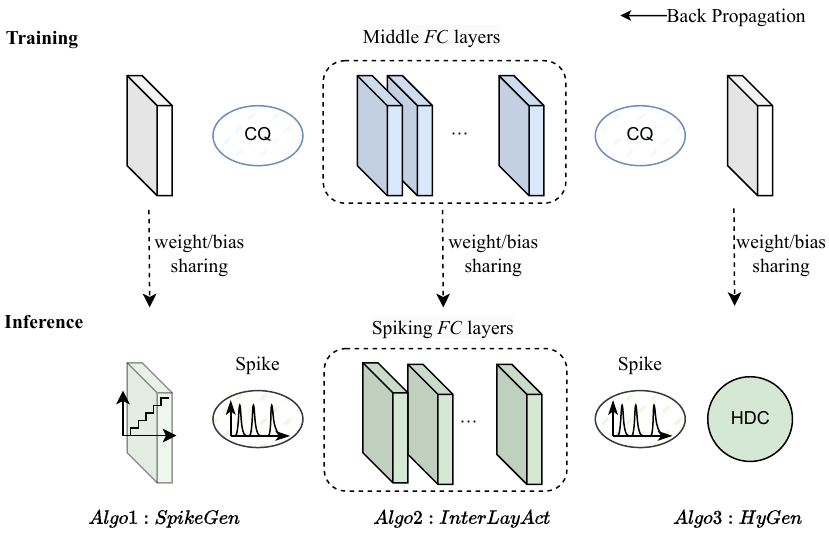}
        \caption{Summary of training and conversion workflow.}
    \label{fig:workflow_}
\end{figure}

Unfortunately, state-of-the-art SNNs often require high latencies for highly accurate detection, therefore increasing energy consumption and possibly nullifying their energy benefits over ANNs~\cite{10152465,deng2021optimal,yan2021near}. In this research, through a refined spike model, we reduce the inference latency to just 1 to 4 timesteps, significantly reducing energy consumption. On the other hand, state-of-the-art HDC utilizes large vectors (5000 to 10,000 dimensions) to boost accuracy due to enhanced orthogonality. This also leads to higher energy consumption~\cite{neubert2019introduction, ge2020classification}. We shall show that by integrating HDC with the binary output vector of SNNs and subsequent
fine-tuning, we can trim vector dimensions significantly, from thousands to mere tens or hundreds, while keeping the number of timesteps of the SNN small. Consequently, we propose a novel model we call {\em HyperSNN} that blends low-latency SNNs with reduced dimension HDC, resulting in a method for control that is energy efficienct and yet robust and accurate.

The workflow of HyperSNN is shown in Figure~\ref{fig:workflow_}. HyperSNN integrates {\em clamp and quantization} functions during training to minimize the ANN-to-SNN conversion's information loss. Within the embedding layer, the output is transformed into a spike train during inference using the {\em integrate-and-fire} model. The spike train's generation threshold is adjusted to coincide with the quantization level of input data, weights, and biases, amplifying neural activity~\cite{yan2022low}. For the intermediate layer, we employ the clamp and quantization function (CQ) during MLP training. At inference, trained weights and biases are used for SNN conversion. Our model also introduces a classification layer, leveraging HDC in place of traditional layers to perform similarity checks. In the Cartpole setup, our quantized models show a 9.5\% mean noise resilience boost and significant reduction in energy consumption. Specifically, our setup — with an 8-bit input, SNN for primary and intermediate layers, and HDC for output — achieves baseline performance using just 23.10243 pJ of energy, or 9.96\% of the MLP baseline's energy. This trend is consistent in various AI gym environments\footnote{https://www.gymlibrary.dev/index.html}, including Acrobot, MountainCar, and Lunar Lander. Notably, in the Lunar Lander setting, our "8bits input+8bits SNN (T = 4)" configuration attains energy efficiency at 0.6897 J (just 3.8\% of the original 18.1148J MLP model), exhibiting resilience in different noise conditions, unlike the MLP baseline. This efficiency enhancement encourages the adoption of advanced control algorithms like MPC, setting the stage for more sophisticated techniques.

Our proposed control systems prioritize a harmonious balance between control accuracy, power efficiency, and robustness, particularly for applications in intelligent furniture and robotics. Such balance not only elevates user experience but also fosters wider adoption and sustainability of these innovations. Fundamentally, our research offers some solutions to control challenges prevalent in interactive, mobile, wearable, and ubiquitous tech domains. The key contributions of this paper are:

\begin{itemize}
    \item We introduce hyper SNN, a unique spiking neural network (SNN) model combined with hyperdimensional computing (HDC), a novel integration in the control domain to our understanding.\footnote{https://anonymous.4open.science/r/SnnControl-926E/README.md} Our methodology, by confining HDC's dimensionality to tens or hundreds and utilizing an SNN timestep between 1 to 4, achieves benchmark-equivalent accuracy. Remarkably, it does so while slashing energy demands to just 1.36\% to 9.96\% of traditional methods, coupled with enhanced robustness.
    
    \item We undertook a comprehensive examination of the synergy between SNN and HDC in driving power efficiency and robustness. Our study spanned four canonical AI gym environments and assessed the effects of four specific noise types, each relevant to distinct applications.

    \item We examine the feasibility of applying our model to complex control algorithms, notably model predictive control (MPC), observing a marked increase in robustness. We also delve into its utility for classification tasks relevant to wearable technology applications, encompassing human activity recognition, speech analysis, and image classification.

\end{itemize}

\section{Methods}

We introduce HyperSNN for control challenges, prioritizing increased robustness, reduced latency, and maintained or improved accuracy. The detailed workflow of our model is depicted in Figure~\ref{fig:workflow_}. As evident from the figure, the embedding layer is first trained using clamp and quantization functions~\cite{10152465}, subsequently converting its output into a spike train via the IF model. Intermediate layers are initially trained in an MLP framework, later transitioning to an SNN during inference. This transformation replaces energy-intensive 32-bit floating point (FP32) multiplications with efficient 8-bit integer (INT8) addition operations. In the concluding classification layer, traditional fully connected layers are replaced by hyper-dimensional computing, using similarity checks to efficiently compute distances and generate results.

\subsection{Quantized activation}

In this section, we discuss the benefits of quantized activation from a theoretical perspective.

Considering the basic linear activation, denoted as $\sigma(z)=z$, it's evident that $\frac{Var (\sigma(x))}{Var (x)} = 1$, suggesting that linear activation processes input noise without inherent robustness. For our discussion, we assume \( x \) follows a Gaussian distribution, with noise defined by a zero mean and finite variance. This mirrors the behavior of truncation errors, where an input affected by such an error sees its truncated value consistently spread within the range $[-\delta, \delta]$.

In the context of Relu activation, given the symmetric property of the Gaussian input around 0, its variance ratio is deduced as $$\frac{Var (\sigma_{\textrm{Relu}}(x))}{Var (x)} = \frac{1}{2}.$$


Moving to the activation function (spike rate) of SNN, defined over the range $[0,1/T, 2/T,...,(T-1)/T, 1]$, we illustrate its resilience to input disturbances through its variance ratio: $$\frac{Var (\sigma_{\textrm{SNN}}(x))}{Var (x)} < \frac{Var( \sigma_{\textrm{Relu}}(x))}{Var (x)}.$$

The proof idea is demonstrated in Figure~\ref{Energy1}, ~\ref{Energy2} and ~\ref{Energy3}. Here $V$ denotes the variance of the output after the activation layer. 

\begin{figure*}[ht] 

\begin{minipage}[t]{0.33\linewidth} 
\centering
\includegraphics[width = 0.9\linewidth]{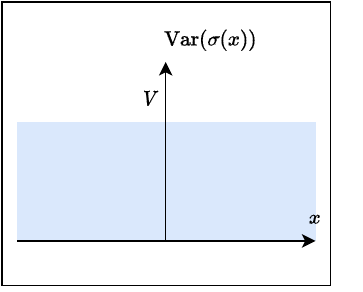}
\caption{Linear activation function} 
\label{Energy1} 
\end{minipage}
\begin{minipage}[t]{0.33\linewidth}
\centering
\includegraphics[width = 0.9\linewidth]{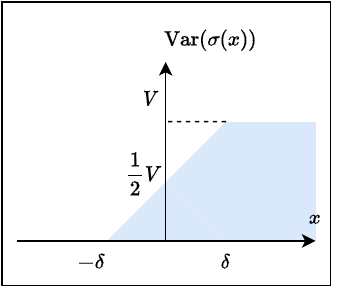}
\caption{Relu activation function}
\label{Energy2}
\end{minipage}
\begin{minipage}[t]{0.33\linewidth} 
\centering
\includegraphics[width = 0.9\linewidth]{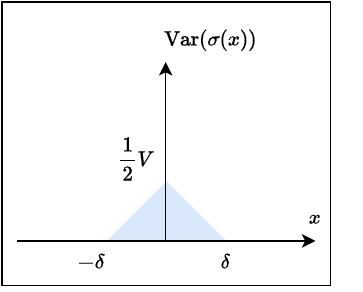}
\caption{SNN activation (spike rate) function} 
\label{Energy3} 
\end{minipage}

\end{figure*}

\subsection{Robustness for quantization weights and SNN activations }


Given a parameterized continuous function $f_{\theta}$, without loss of generality, assume it is $L$-Lipschitz continuous:
\begin{equation}
    \|f_{\theta_1}(x) - f_{\theta_2}(x)\| \leq L \|\theta_1 - \theta_2\|_2.
\end{equation}

In this paper, we assume the target function $f$ to be robust against perturbations. Stable targets can be approximated and optimized using parameterized models, as demonstrated by Wang et al.~\cite{wang2023inverse}.

As is shown in Figure 1, denote the quantization level for the model by $\Delta \theta$, the stability of a model against weight perturbation is equivalent to
\begin{equation}
    |f_{\theta}(x) - f_{Q(\theta)} (x)| \leq w \left ( \| \theta - Q(\theta) \| \right ), \quad \forall x, \theta. 
\end{equation}
Here, \( Q(\theta) \) denotes the quantized model weights, ensuring \( \|Q(\theta) - \theta \|_\infty \leq \Delta \theta \). The function \( w(\cdot) \) represents the modulus of continuity for model \( f \) concerning its weights.

\begin{proposition}[Robustness against weight perturbation]
    Assume the optimal target function is $L$-Lipschitz continuous. Let the model weights be $D$-dimensional, 
    \begin{equation}
        |f_{\theta}(x) - f_{Q(\theta)} (x)| \leq D L \|\Delta \theta\|_\infty. 
    \end{equation}
    In other words, $w(\Delta \theta) \leq D L \|\Delta \theta\|_\infty $.
\end{proposition}

Next, we consider the binary classification task, which can be generalized into more general settings such as multi-class classification. 

\begin{proposition}[Robustness against input perturbation]
    Assume an input is bounded away from decision line $\ell: w^T x + b = 0$, with distance $d$. 
    For inputs noise $\Delta x$ small enough ($(w^T x + b)(w^T (x+ \Delta x) + b) >0$), there is a quantization method $Q$ such that
    \begin{equation}
        \left (w^{*\top} x + b^* \right ) \left (Q(w)^T (x+\Delta x)+ Q(b) \right ) \geq 0
    \end{equation}
    In other words, the quantized model is robust against input perturbation when the input is bounded away from 0. 
\end{proposition}

\begin{proof}
    Without loss of generality, assume $w^\top x + b > 0$, therefore $(w^\top (x+ \Delta x) + b) >0$. 
    The inputs are normalized such that $\|x\| \leq 1$.
    Moreover, let the quantization scale be small enough such that $\max(\|Q(w) - w\|, \|Q(b) - b\|) \leq \frac{1}{2} (w^\top (x+ \Delta x) + b).$
    It can be seen 
    \begin{equation}
        Q(w)^T (x+\Delta x)+ Q(b) \geq 0.
    \end{equation}

    Therefore the quantized model is robust against the inputs noise. 
\end{proof}

As can be seen in Section ~\ref{experiments_}, the low-dimension model can also potentially improve the robustness and efficiency of the model. 

\subsection{Embedding layer}

In this section, we detail the procedure for converting sensor input data to achieve robustness and energy efficiency.

Sensor input data is first converted into an $n$-bit fixed-point integer. Here, the leading bit represents the sign, and the subsequent bits encode the real number. As such, an $n$-bit fixed-point integer captures the range $[-2^{n-1}+1, ..., 2^{n-1}-1]$.

As outlined in Algorithm 1, the input data is first normalized to fit the range $[\dfrac{1}{2^{n-1}}-1, 1-\dfrac{1}{2^{n-1}}]$. It is then mapped to fixed-point integer values. Similarly, the trained weights and biases undergo normalization to this range before their conversion into fixed-point integer values.

To enhance power efficiency and prepare data for the SNN layers, the integer inputs are encoded into spike trains over a time step of $T$, producing only binary values (0 or 1). This encoding is based on a standard input spike generation method cited in this study~\cite{yan2022low}. For an integer input $x$, if the value \(V\) exceeds a set threshold \(\theta\), the system registers a spike ('1') and decreases \(V\) by \(\theta\). Otherwise, it outputs '0'.

\begin{algorithm}[ht]
\caption{``SpikeGen'': embedding layer spike train generation } 
\label{alg1}

\begin{minipage}[t]{0.48\linewidth}
\begin{algorithmic}
\REQUIRE Input data $x$ with max input of $m$; Initial Membrane voltage $V = 0$.
\ENSURE Input spike train $\bm{s}$ of length $T$ (SNN Time step); Threshold $\theta$ for the input SNN layer .
\STATE  $\bm{STEP1}$. \textbf{Input layer (pro-precessing, quantized to $n$ bits)}:
\STATE $\bm{x} = (1-1/(2^{n-1})) * \dfrac{\bm{x}}{m}$
\STATE $\bm{x} = \lfloor \bm{x}\cdot 2^{n-1} \rfloor$
\STATE $\theta =(1-1/(2^{n-1}))*\dfrac{2^{n-1}}{m} $ (define Initial theta for Step 2 and 3)
\STATE
\STATE $\bm{STEP2}$. \textbf{Weight/bias quantization with $q$ level} 
\STATE $f$ = max\{$\bm{w}$,$\bm{bias}$\}
\STATE $\bm{w} = (1-1/(2^{q-1})) \cdot \dfrac{\bm{w}}{f}$
\STATE $\bm{b} = (1-1/(2^{q-1})) \cdot \dfrac{\bm{b}}{f}$

\end{algorithmic}
\end{minipage}\hfill
\begin{minipage}[t]{0.48\linewidth}
\begin{algorithmic}
\STATE $\bm{w} = \lfloor \bm{w}\cdot 2^{q-1} \rfloor$
\STATE $\bm{b} = \lfloor \bm{b}\cdot 2^{q-1} \rfloor$
\STATE $\theta = \theta * (1-1/(2^{q-1}))*\dfrac{2^{q-1}}{f}$
\STATE
\STATE $\bm{STEP3}$. \textbf{Input spike train generation:}
\FOR{ $i= 0$ to $T-1$}
\STATE $V = V+(\bm{w}*x+\bm{b})$
\IF{$V \geq  {\theta}$} 
\STATE $\bm{s}_i = 1$
\STATE $V = V - \theta$
\ELSE
\STATE $\bm{s}_i = 0$ // The shape of $\bm{s}$ should then be (x.shape,T) 
\ENDIF
\ENDFOR
\end{algorithmic}
\end{minipage}

\end{algorithm}

The mapping process details are depicted in Figure ~\ref{fig:workflow_1}. It's pivotal to adjust the threshold \(\theta\) based on the quantization levels of the input data \(x\) and weights and biases. As per ~\cite{yan2022low}, given an input \(x\) scaled by \((1-1/(2^{n-1}))*2^{n-1}/m\), the threshold \(\theta\) should be proportionally adjusted. Here, \(m\) is the input normalization, \( (1-1/(2^{n-1})) \) is the scaling factor, and \( 2^{n-1} \) handles the decimal-to-integer conversion. Similarly, during weight quantization, the threshold \(\theta\) in the SNN neuron must be adjusted in a corresponding manner. Notably, both the scaling of weight/bias and the adjustment of the threshold can be conducted prior to inference.

\begin{figure*}[ht]
    \centering
    \includegraphics[width = 0.8\linewidth]{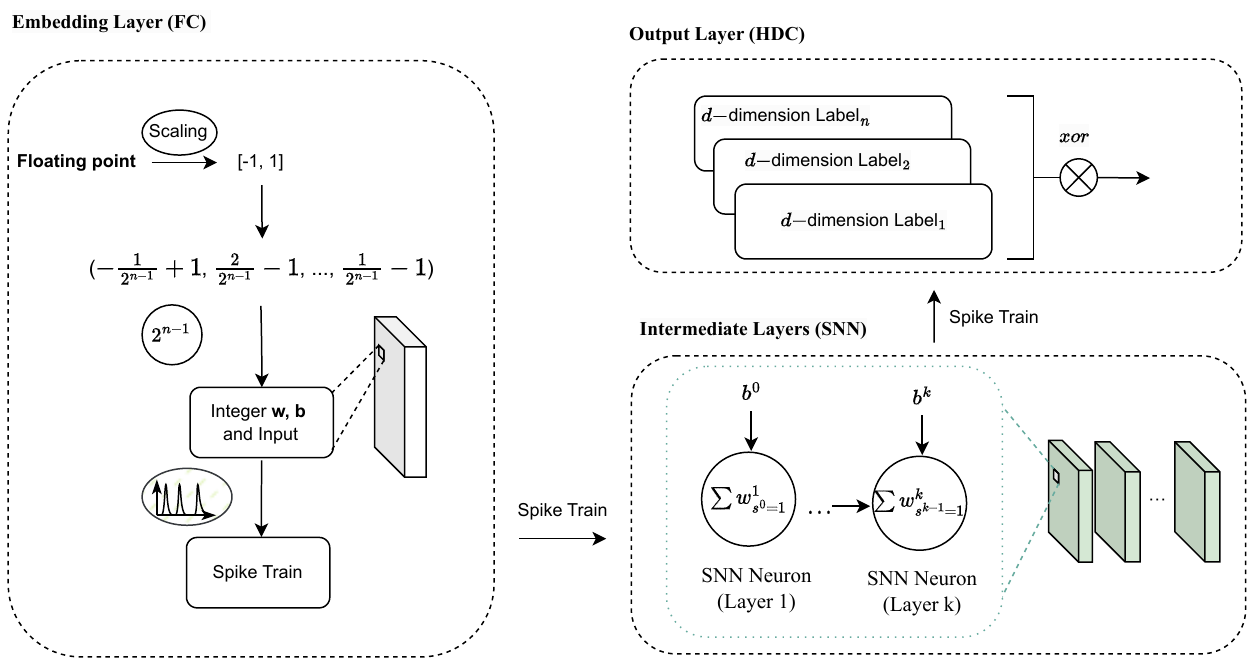}
        \caption{Inference of HyperSNN.}
    \label{fig:workflow_1}
\end{figure*}

\subsection{Intermediate layers}

In our architecture, we replace the commonly used fully connected layers in intermediary stages with a spiking neural network (SNN). This modification shifts from the resource-intensive FP32 multiplications to the more efficient INT8 additions. During training, to reduce the MLP-to-SNN conversion loss, the activation function in fc layers is substituted with a clamping and quantization function. For inference, these layers transition to SNNs, with the clamping and quantization function replaced by the averaging IF model~\cite{yan2022low}. The input spike train, derived from Algorithm~\ref{alg1} step 3, is fed into the SNN, producing outputs via the averaging IF model.

Algorithm~\ref{alg3} details the process: for each layer \(l\), the membrane potential is computed as \(\sum_{j} (w^l_{ s^{l-1}=1}+b^l)\) at each timestep. This is then averaged over \(T\) to obtain \(V^l\), which is passed to the standard Integrate-and-Fire (IF) model. At every timestep \(t\), \(V^l\) from previous steps is integrated to compute \(U^l\). When \(U^l\) exceeds a threshold \(\theta\), a spike (denoted by '1') is produced, and \(V_i^l\) decreases by \(\theta\). Otherwise, the output remains '0'.
\begin{algorithm}[ht]
\caption{``InterLayAct'': SNN activation function among intermediate layers.} 
\label{alg3}
\begin{algorithmic}[1]
\REQUIRE Quantized weights $\bm{w^l}$ at layer $l$ with quantization level $q$; Input spikes $s^{l}$  at layer $l$; Quantized bias $b^l$ at layer $l$ with quantization level of $q$; Membrane potential $U^l$ of at layer $l$ (initialize to 0). Default threshold $\theta$.
\STATE \textbf{STEP1: compute the average}
\STATE $V^l  = 0$
\FOR{$t=1$ to $T$}
\STATE $V^l += \sum (\dfrac{w^l_{ s_t^{l-1}=1}}{T})+ \dfrac{b^l}{T}$ // Division of $T$ can be merged into previous weights and biases
\ENDFOR

\STATE
\STATE \textbf{STEP2: IF Model}
\STATE $ U^l(0) = 0$;
\FOR{ $i= 0$ to $T-1$}
\STATE $U^l = V^l+U^{l-1}$
\IF{$U^l >=  {\theta}$} 
\STATE $\bm{s^l}_i = 1$
\STATE $U^l = U^l - \theta$
\ELSE
\STATE
$\bm{s^l}_i = 0$
\ENDIF
\ENDFOR

\STATE \textbf{return} $s^l$ - the spike train of of layer $l$. 

\end{algorithmic}
\end{algorithm}

\subsection{Output layer}

As illustrated in Figure~\ref{fig:workflow_1}, the spike train \( s^k \) produced post-intermediate layers consists exclusively of binary values 0 and 1. This allows for comparison with labels spanning \( c \) classes to identify the nearest match. To establish the label for each class, we adopt the majority rule described in Algorithm~\ref{alg_mr}. We first accumulate the \( N \) \( s^c_k \) values associated with class \( c \), yielding an integer representation \( L_c \) for the class. If an element in \( L_c \) exceeds a threshold \( \theta \), it's assigned a 1; otherwise, it's marked as 0. This procedure yields a binary label. During inference, the Hamming distance determines the nearest label to finalize the output.

\begin{algorithm}[ht]
\caption{``HyGen'': representation hypervector generation}
\label{alg_mr}
\begin{algorithmic}[1]
\REQUIRE $N$ number of spike output $\bm{s^c_k}$;
\ENSURE Pre-defined threshold $\theta$; 
\STATE   $L_c = 0$
\FOR{$i=1$ to $N$}
\STATE $L_c+=s^c_k$
\ENDFOR
\FOR{$i=1$ to $d$}
\IF {$L_c[i]>\theta$}
\STATE $L_c[i]=1$
\ELSE
\STATE $L_c[i]=0$
\ENDIF
\ENDFOR
\end{algorithmic}
\end{algorithm}
For improved energy efficiency, we can streamline the binary representation label by discarding identical items across all representation labels, \(L_c\). Taking the Cartpole environment with Net1 as an instance: the binary label for class1 "Push cart to the left" is [1., 1., 1., -1., -1., 1., 1., 1., -1., 1.], whereas for "Push cart to the right" it is [1., 1., 1., -1., -1., 1., 1., 1., -1., -1.]. Directly employing the Hamming distance for label identification necessitates 20 xor operations. However, by noting that nine out of ten items are congruent between the two labels, we can truncate the length from ten to just one, reducing the identification process to a mere two xor operations.

\section{Experiments}
\label{experiments_}
\subsection{Experiment setup}
Our model HyperSNN has been implemented based on CUDA-accelerated PyTorch version 1.6.0 and 1.7.1 in this paper. We train the model using PyTorch version 1.7.1 and run the SNNs with PyTorch version 1.6.0. The experiments were performed on an Intel Xeon E5-2680 server with 256GB DRAM, a Tesla P100 GPU, and a GeForce RT 3090 GPU, running 64-bit Linux 4.15.0. The model in this paper were implemented and trained using Pytorch on a server equipped with an Nvidia Tesla P100 card and a GeForce RT 3090 GPU. 

\subsection{Network models}

In our study of control problems, we designed and tested four distinct network architectures, each tailored to a specific environment. Utilizing deep Q-Networks (DQN)~\cite{fan2020theoretical} as a foundational approach, we trained each network and subsequently applied layer-wise fine-tuning. This fine-tuning encompassed techniques such as output quantization, weight quantization, and hyperdimensional computing. The networks, labeled as 1 through 4, are optimized for the Cartpole, Acrobot, Mountain Car, and Lunar Lander tasks, respectively. Notably, any reduction in network size resulted in a substantial decline in accuracy. Thus, our designs consistently featured 2 to 3 layers of spiking fully connected networks (FCNs) and capped the output size at 64. A comprehensive overview of the network configurations employed in this research is provided in Table~\ref{tab:network_com}.

Using Net3 as an example, we shall show how we compute the energy consumption of each layer by referencing energy values from Horowitz's presentation at ISSCC 2014~\cite{horowitz20141} and the analysis presented at CICN 2011~\cite{nishad2011analysis}. While this is fairly old data (we were unable to find similar ones that are newer), it is worthwhile to point out that (i) smart devices do not often use the latest and greatest silicon technology due to cost reasons and (ii) the relative difference in magnitude of the values in the comparison is likely to hold regardless of the technology used. With these in mind, the following energy values are used:
\begin{itemize}
    \item FP32 multiplications: 3.7pJ; 32-bit integer multiplications: 3.1pJ; INT8 multiplications: 0.2pJ
    \item FP32 add: 0.9pJ; 32-bit integer add: 0.1pJ; INT8 add: 0.03pJ
    \item Xor operation: 0.00243pJ
\end{itemize}

The energy calculations for the baseline MLP are:
\begin{itemize}
    \item \textbf{Input layer}: 2 $\times$ 24 FP32 multiplications and 24 FP32 additions yield a total energy of 199.2pJ.
    \item \textbf{Intermediate layer}: 24 $\times$ 24 FP32 multiplications and 24 FP32 additions result in an energy of 2152.8pJ.
    \item \textbf{Output layer}: 24 $\times$ 3 FP32 multiplications and 3 FP32 additions give 269.1pJ.
\end{itemize}

For our 8-bit input and 8-bit SNN (T=1) model integrated with HDC:
\begin{itemize}
    \item \textbf{Input layer}: Following Algorithm~\ref{alg1}, input scaling and initial processing necessitate 2 FP32 multiplications, 2 $\times$ 24 INT8 multiplications (display as "48+2" in table~\ref{tab:network_mountaincar}), and 24 INT8 additions. This amounts to 17.72pJ (spike train generation energy is excluded with T=1).
    \item \textbf{Intermediate layer}: The energy is computed for 24 $\times$ 24 $\times$ spike rate INT8 additions and 24 INT8 additions. Taking a conservative approach with a spike rate of 1, the total is 18pJ (excluding average membrane potential with T=1).
    \item \textbf{Output layer}: The energy for 24 $\times$ 3 xor operations totals 0.05832pJ.
\end{itemize}

\subsection{Control problem analysis}
For the control problem, we employ four renowned, publicly accessible classic control environments from AI gyms, ranging from the simple Cartpole and Acrobot, to the more complex MountainCar and Lunar Lander for simulations. Our models are validated by considering three crucial metrics: rewards, operation count, and robustness. 

\subsubsection{Cartpole}
The environment emulates the cart-pole problem described by Barto, Sutton, and Anderson~\cite{barto1983neuronlike}. In this setting, a pendulum attached to a cart via an unactuated joint allows the cart to move on a frictionless track. The goal is to maintain the pole upright by applying left or right forces to the cart. Success is indicated by a longer duration of balanced pole. For each timestep the pole stays upright, the agent is rewarded with +1. An episode concludes if the pole deviates over 15 degrees from the vertical position or the cart shifts beyond 2.4 units from the center.

\begin{figure*}[ht] 

\begin{minipage}[t]{0.49\linewidth} 
\centering
\includegraphics[width = 0.8\linewidth]{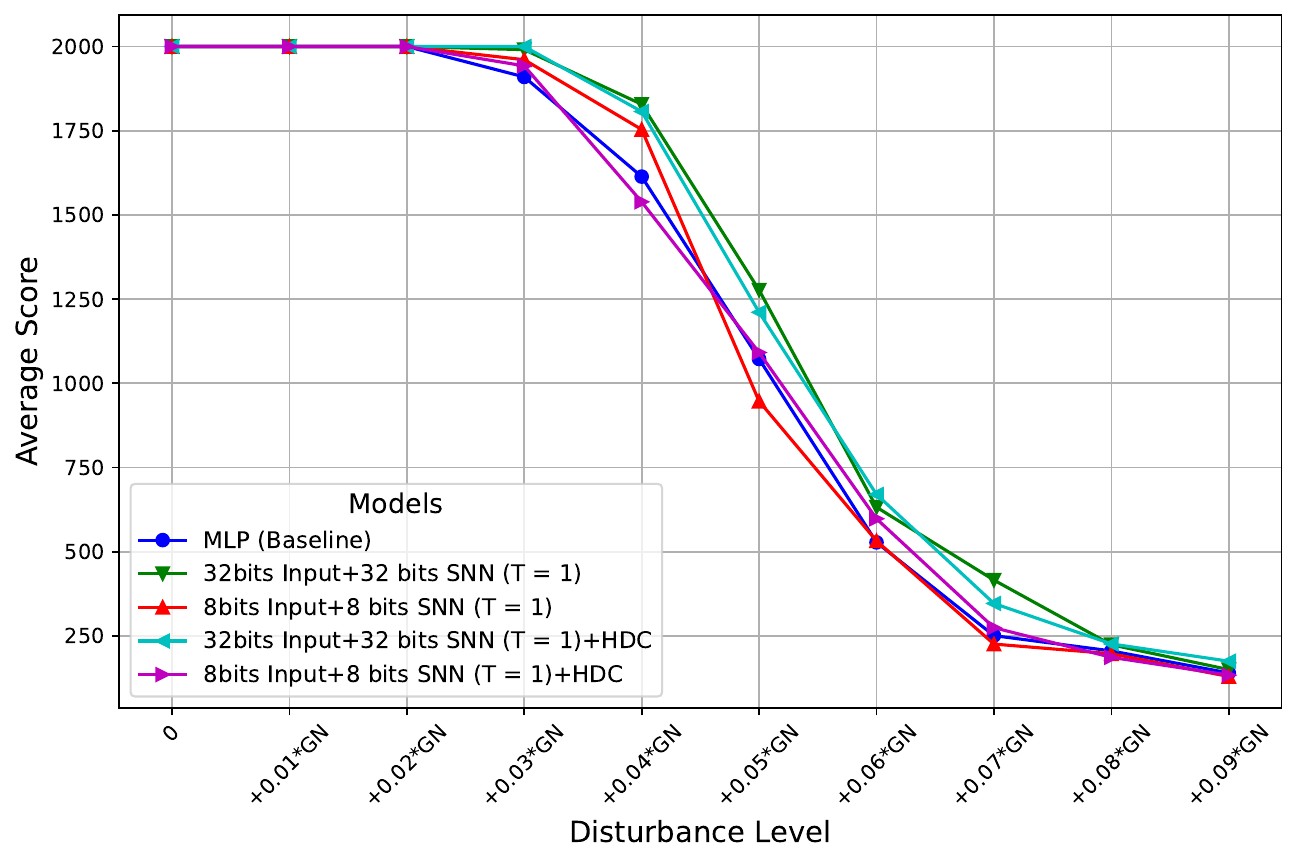}
\caption{Performance under Gaussian noise (Cartpole)} 
\label{robust_cartpole_0} 
\end{minipage}
\hfill
\begin{minipage}[t]{0.49\linewidth}
\centering
\includegraphics[width = 0.8\linewidth]{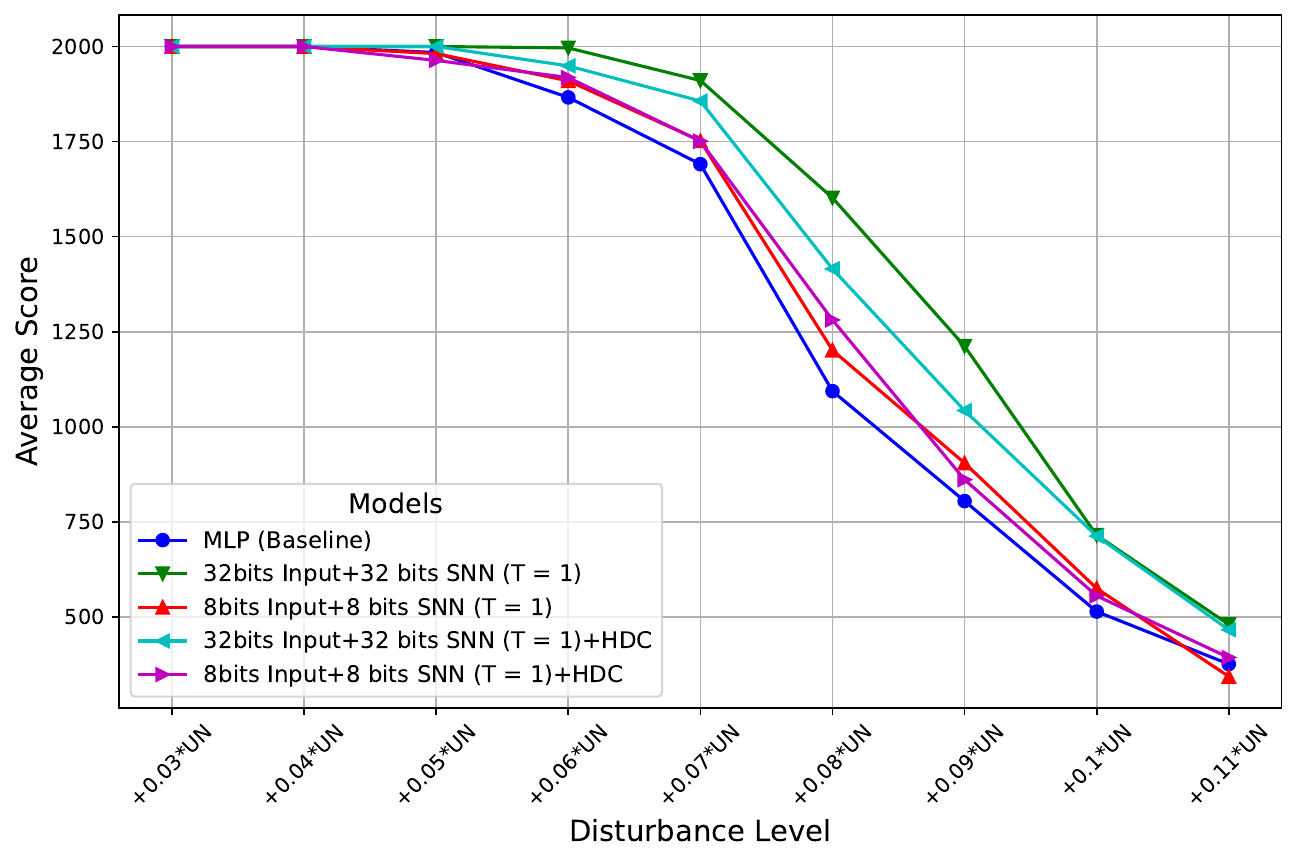}
\caption{Performance under Uniform noise (Cartpole)}
\label{robust_cartpole_1}
\end{minipage}
\vspace{4ex}

\begin{minipage}[t]{0.49\linewidth} 
\centering
\includegraphics[width = 0.8\linewidth]{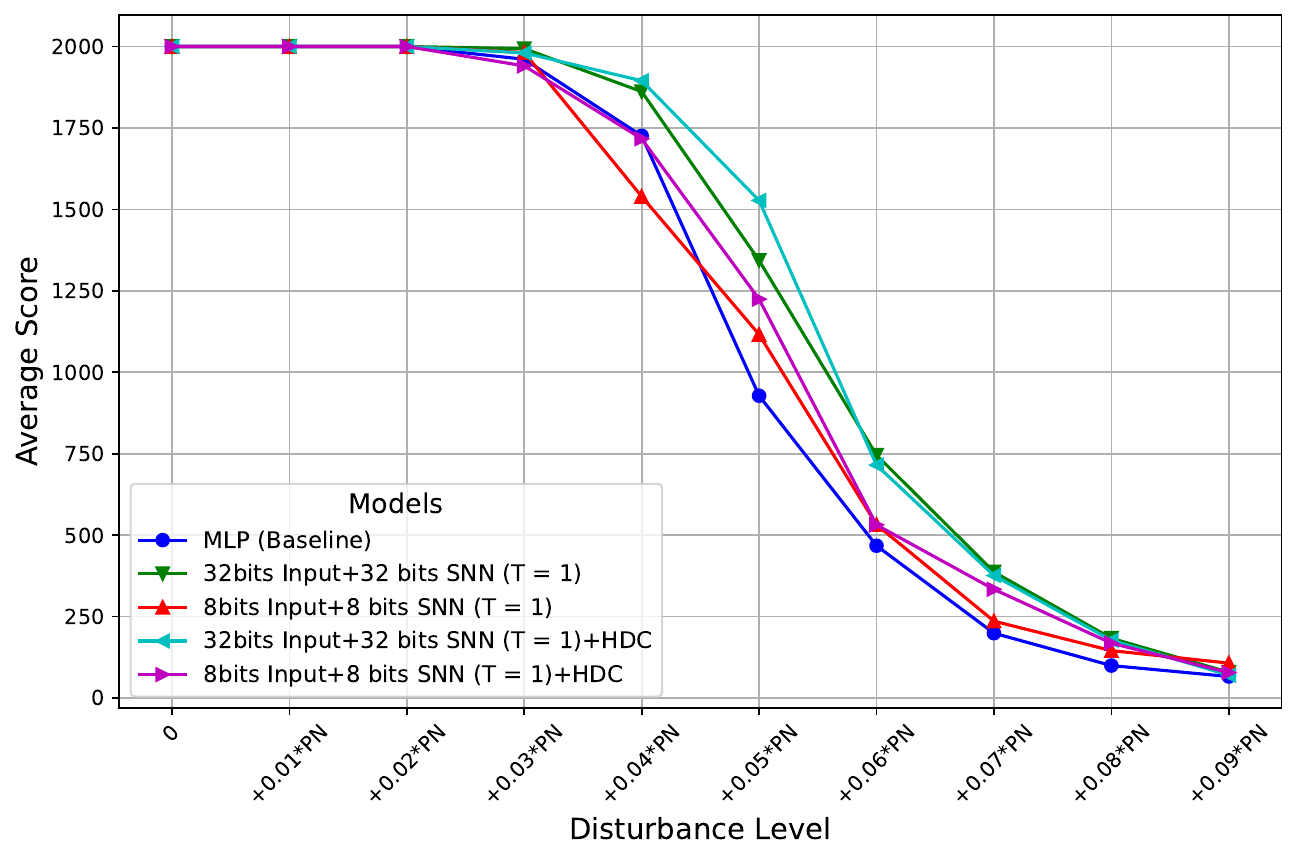}
\caption{Performance under Poisson noise (Cartpole)} 
\label{robust_cartpole_2} 
\end{minipage}
\hfill
\begin{minipage}[t]{0.49\linewidth}
\centering
\includegraphics[width = 0.8\linewidth]{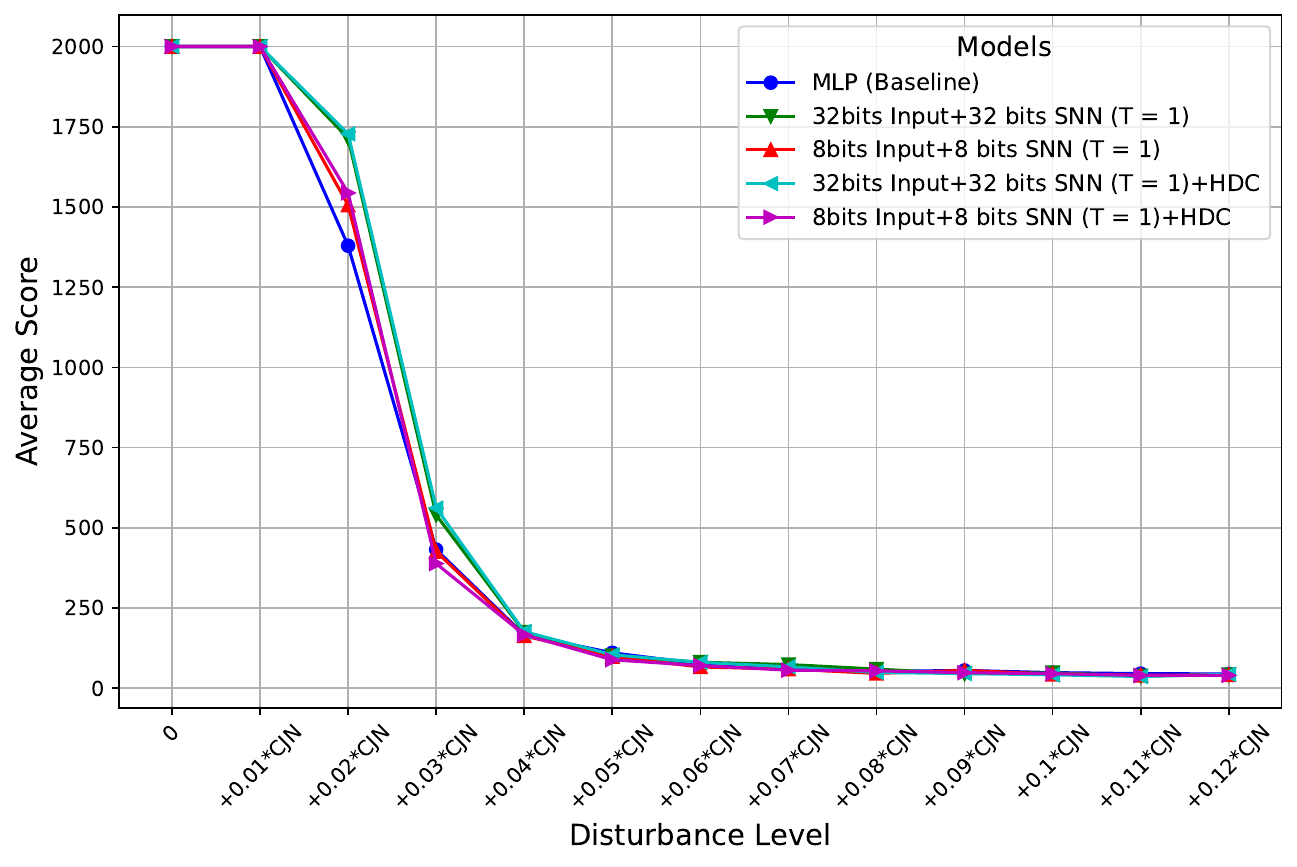}
\caption{Performance under Clock jitter noise (Cartpole)}
\label{robust_cartpole_3}
\end{minipage}

\end{figure*}

We assess the performance of our models using the Net1 network structure across varied weight and input quantization levels, both with and without hyperdimensional computing (HDC). For comparison, we employ a pure MLP model with the same architecture. While the AI gym environment marks success at 500 steps, we extend our evaluation to 2000 steps to observe the cartpole's balance more comprehensively. We conducted each experiment under 100 distinct initial conditions and averaged the results to bolster reliability. As delineated in Table~\ref{tab:network_cart}, our quantized models consistently matched the baseline MLP's performance but with reduced energy consumption. Notably, our model with 8-bit quantized input and weight/bias in the input and intermediate layers (SNN), coupled with HDC at the output, demonstrated high performance analogous to the baseline. However, its energy consumption was just 23.10243 pJ—only 9.96\% of the energy requirement of the MLP baseline.

\begin{table}[ht]
\small
\begin{threeparttable}

\begin{tabular}{c|ccc|cclc|l}
\hline
\multirow{2}{*}{Cartpole}                                                                             & \multicolumn{3}{c|}{Input layer}                                                     & \multicolumn{4}{c|}{Intermediate/Output layer}                                                                                    & \multirow{2}{*}{Rewards} \\ \cline{2-8}
                                                                                                     & \multicolumn{1}{c|}{Adds} & \multicolumn{1}{c|}{Mults} & \multicolumn{1}{c|}{Energy} & \multicolumn{1}{c|}{Adds}           & \multicolumn{1}{c|}{Mults}        & \multicolumn{1}{c|}{Bool} & \multicolumn{1}{c|}{Energy} &                          \\ \hline
\begin{tabular}[c]{@{}c@{}}MLP (baseline)  (FP32)\end{tabular}                            & \multicolumn{1}{c|}{10}    & \multicolumn{1}{c|}{40}    & 157pJ                       & \multicolumn{1}{c|}{2}              & \multicolumn{1}{c|}{20 } & \multicolumn{1}{c|}{0}    & 75pJ                        & 2000                     \\ \hline
\begin{tabular}[c]{@{}c@{}}32bits input+ 32bits SNN (T = 1)\\ (INT32)\end{tabular}      & \multicolumn{1}{c|}{10}    & \multicolumn{1}{c|}{40+4}    & 139.8pJ                       & \multicolumn{1}{c|}{22 } & \multicolumn{1}{c|}{0}            & \multicolumn{1}{c|}{0}    & 2.2pJ                      & 2000                     \\ \hline
\begin{tabular}[c]{@{}c@{}}8bits input+ 8bits SNN (T = 1)\\ (INT8)\end{tabular}         & \multicolumn{1}{c|}{10}    & \multicolumn{1}{c|}{40+4}    & 23.1pJ                       & \multicolumn{1}{c|}{22 } & \multicolumn{1}{c|}{0}            & \multicolumn{1}{c|}{0}    & 0.66pJ                      & 2000                     \\ \hline
\begin{tabular}[c]{@{}c@{}}32bits input+ 32bits SNN (T = 1) +HDC\\ (INT32)\end{tabular} & \multicolumn{1}{c|}{10}    & \multicolumn{1}{c|}{40+4}    & 139.8pJ                         & \multicolumn{1}{c|}{0}              & \multicolumn{1}{c|}{0}            & \multicolumn{1}{c|}{1}    & 2.43fJ                       & 2000                     \\ \hline
\begin{tabular}[c]{@{}c@{}}8bits input+ 8bits SNN (T = 1)+HDC\\ (INT8)\end{tabular}     & \multicolumn{1}{c|}{10}    & \multicolumn{1}{c|}{40+4}    & 23.1pJ                       & \multicolumn{1}{c|}{0}              & \multicolumn{1}{c|}{0}            & \multicolumn{1}{c|}{1}    & 2.43fJ                       & 2000                     \\ \hline
\end{tabular}
\caption{Comparison of performance and energy consumption (Cartpole).}
\label{tab:network_cart}

\begin{tablenotes}
\item {
In our convention, ``\(x\)-bits input + \(x\)-bits SNN'' denotes employing \(x\)-bit input, weight, and bias in the input layer, and \(x\)-bit weight and bias in the intermediate layer. The descriptors ``with HDC'' and ``without HDC'' indicate the incorporation of HDC or SNN in the terminal output layer, respectively. A reward of \(x\) represents the cartpole's ability to sustain balance for \(x\) steps prior to faltering.
 }
\end{tablenotes}
\end{threeparttable}
\end{table}

In scenarios such as wearable devices and smart home systems, the limitations of hardware often necessitate the use of low-sensitivity sensors. It is imperative to evaluate our model's robustness to noise in these settings, particularly in comparison to traditional MLP models.

We introduce four types of noise in our assessment:

\begin{itemize}
    \item \textbf{Gaussian noise (GN)}: Referenced on the x-axis of Figure~\ref{robust_cartpole_0}, this noise has a mean of 0 and a variance of 1. It serves as a versatile representation of real-world disturbances like measurement errors or electronic thermal noise. Given the Cartpole's average input of [0.04, 0.11, 0.003, 0.16], we incorporate Gaussian noise scaled by \(k \times 0.01\), with \(k \in [0, 9]\) for our evaluation.
    
    \item \textbf{Uniform noise (UN)}: Indicated on the x-axis of Figure~\ref{robust_cartpole_1}, this noise varies uniformly between -1 and 1. It is suitable when the exact nature of noise is indeterminate but confined within a known range. We introduce Uniform noise scaled by \(k \times 0.01\) (for \(k \in [3, 11]\) in our study. Results with \(k<3\) consistently achieved maximum scores.
    
    \item \textbf{Poisson noise (PN)}: Displayed on the x-axis of Figure~\ref{robust_cartpole_2}, this noise, also known as shot noise, stems from the variability in discrete particle events, like those in imaging devices. Using the Cartpole's average input of approximately 0.1, we integrate Poisson noise scaled by \(k \times 0.01\), where \(k \in [0, 9]\), into our evaluation.
    
    \item \textbf{Clock jitter noise (CJN)}: Highlighted on the x-axis of Figure~\ref{robust_cartpole_3}, this noise simulates errors from timing inaccuracies in digital systems, often stemming from hardware clock instabilities. In our study, we use Clock jitter noise with a standard deviation of \(k \times 0.01\) noise for \(k \in [0, 8]\).
\end{itemize}

Results spanning Figure~\ref{robust_cartpole_0} to Figure~\ref{robust_cartpole_3} underscore the superior noise resilience of our quantized models, averaging an enhancement of 9.5\% across varied noise types and intensities.

\textbf{Gaussian noise (GN)}: Up to +0.02GN, the models perform comparably. From +0.03GN onwards, the `32bits input+32bits SNN (T = 1)+HDC' model outperforms the baseline by around 18.56\%. At +0.07GN, the `8bits input+8bits SNN (T = 1)+HDC' variant exceeds the baseline by 9.4\%.

\textbf{Uniform noise (UN)}: All models demonstrate consistent behavior until +0.04UN. Beyond this threshold, the `8bits input+8bits SNN (T = 1)+HDC' and `32bits input+32bits SNN (T = 1)+HDC' models achieve robustness improvements of 6.0\% and 19.5\%, respectively. Notably, at +0.08UN, they surpass the baseline by approximately 17\% and 29\%. Across all noise intensities, each quantized SNN variant, with or without HDC, consistently eclipses the baseline.

\textbf{Poisson noise (PN)}: Performance remains uniform across all models up to +0.02PN. Between +0.03PN and +0.09PN, the `8bits input+8bits SNN (T = 1)+HDC' and `32bits input+32bits SNN (T = 1)+HDC' models stand out, achieving robustness surges of 28.39\% and 43.05\% over the baseline, respectively. Specifically, at +0.05PN, these metrics climb to 31.93\% and a striking 64.58\%. Similar to the UN scenario, all quantized SNN configurations consistently outshine the baseline, irrespective of HDC utilization.

\textbf{Clock jitter noise (CJN)}: Model performance begins to decline from +0.02CJN, with none enduring beyond +0.04CJN. As a result, our evaluations are restricted between 0 CJN and +0.04CJN. Within this range, the `32bits input+32bits SNN (T = 1)+HDC' model frequently exhibits a robustness advantage, averaging 12.28\% and peaking at 30\% over the baseline at +0.03CJN. Despite occasional variances across noise profiles, the synergy of quantized SNN and HDC consistently offers tangible benefits over the baseline in noise robustness.

\begin{figure*}[ht] 

\begin{minipage}[t]{0.49\linewidth} 
\centering
\includegraphics[width = 0.8\linewidth]{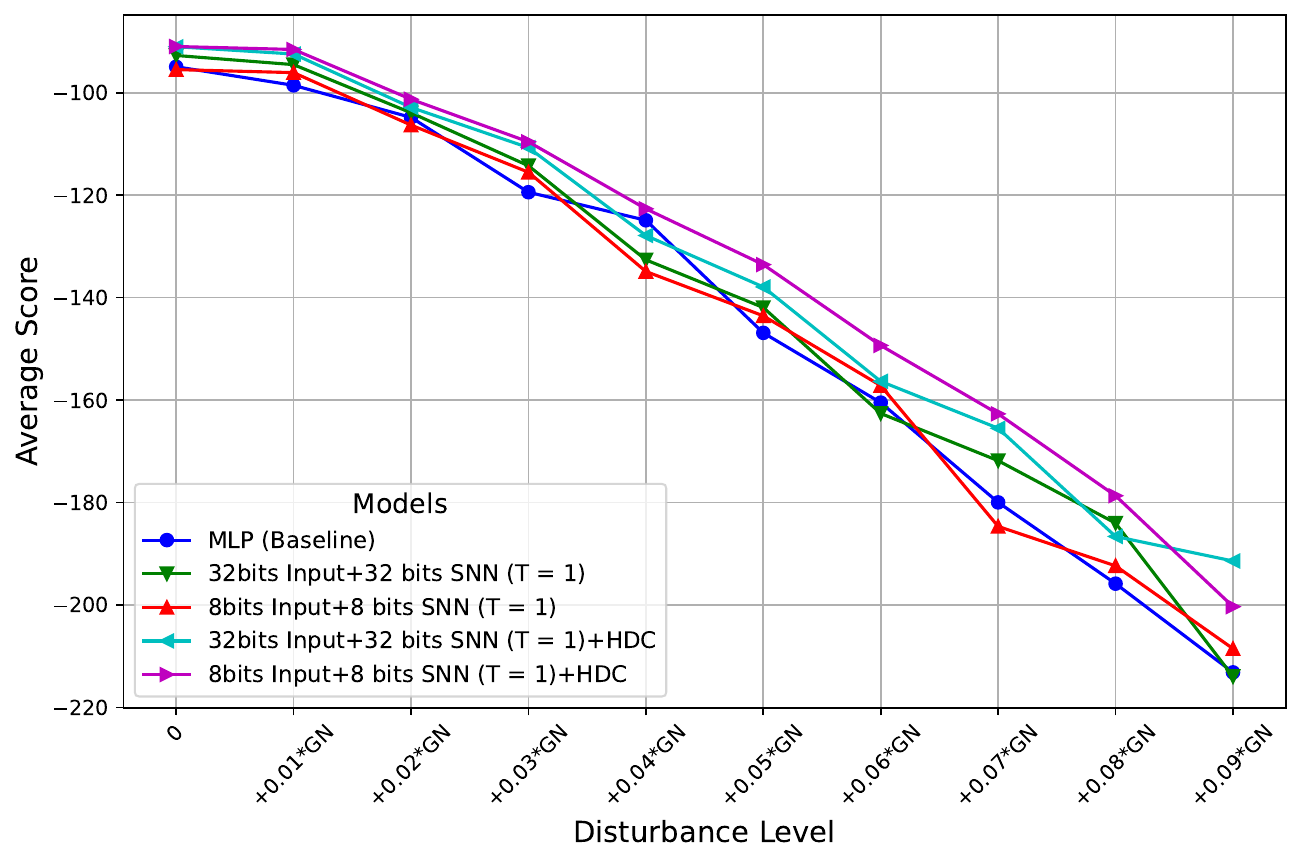}
\caption{Performance under Gaussian noise (Acrobot)} 
\label{robust_acrobot_0} 
\end{minipage}
\hfill
\begin{minipage}[t]{0.49\linewidth}
\centering
\includegraphics[width = 0.8\linewidth]{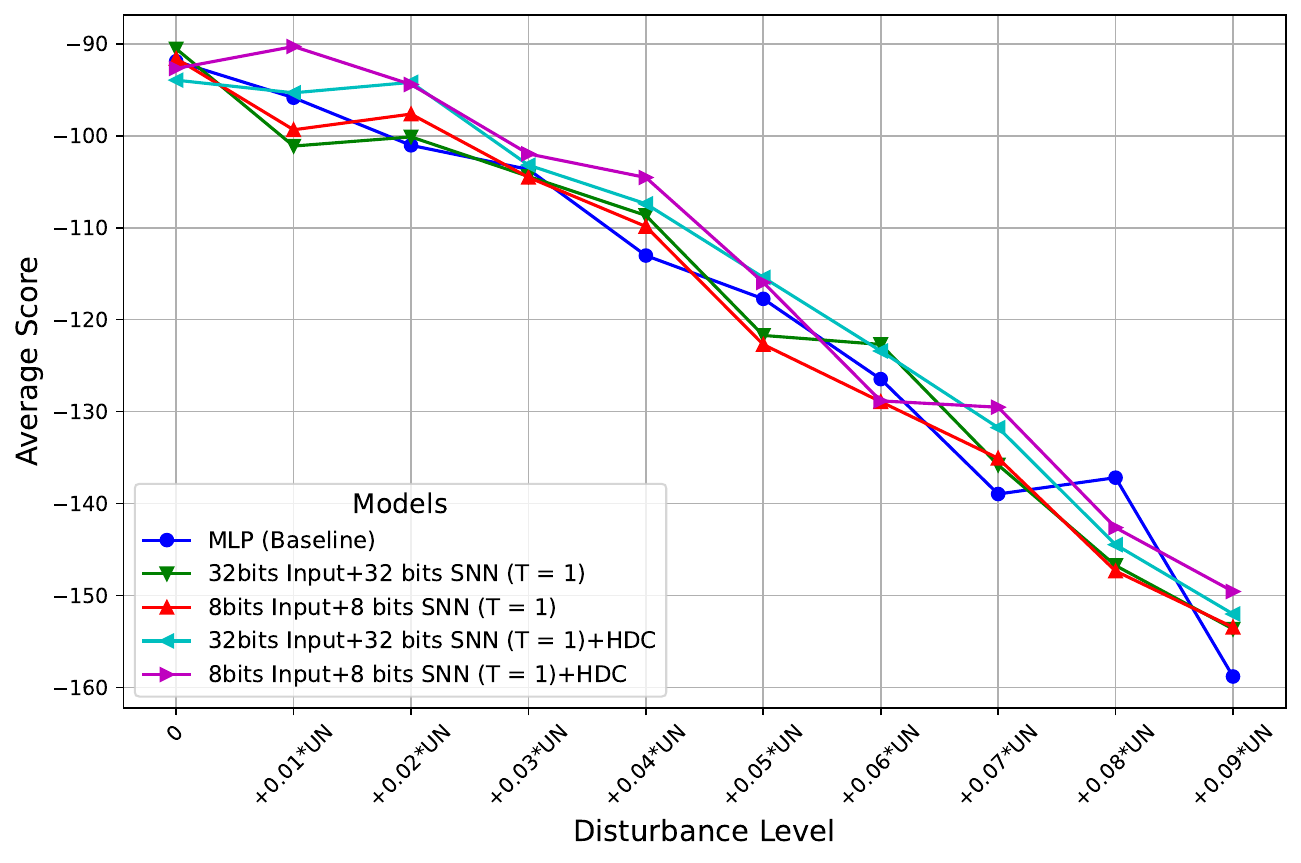}
\caption{Performance under Uniform noise (Acrobot)}
\label{robust_acrobot_1}
\end{minipage}
\vspace{4ex}

\begin{minipage}[t]{0.49\linewidth} 
\centering
\includegraphics[width = 0.8\linewidth]{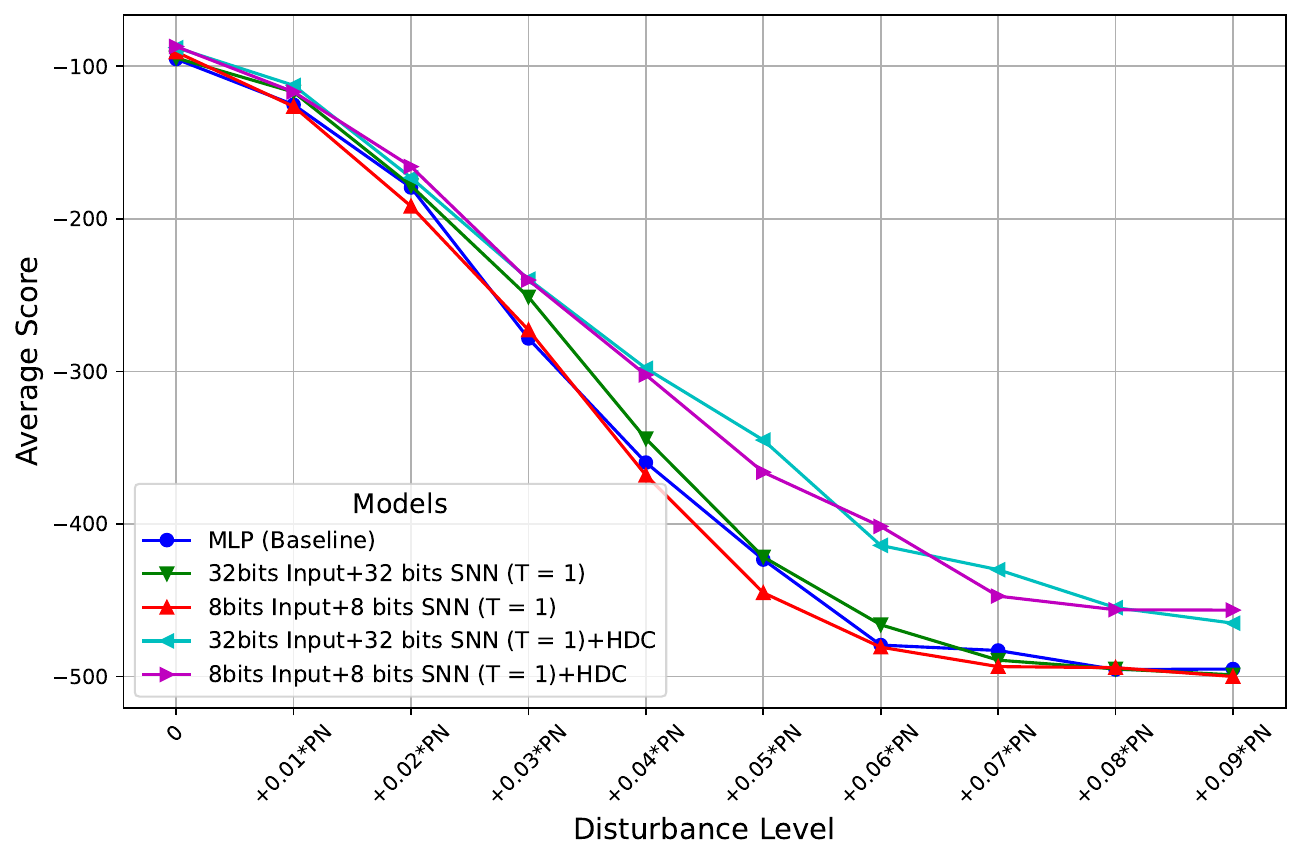}
\caption{Performance under Poisson noise (Acrobot)} 
\label{robust_acrobot_2} 
\end{minipage}
\hfill
\begin{minipage}[t]{0.49\linewidth}
\centering
\includegraphics[width = 0.8\linewidth]{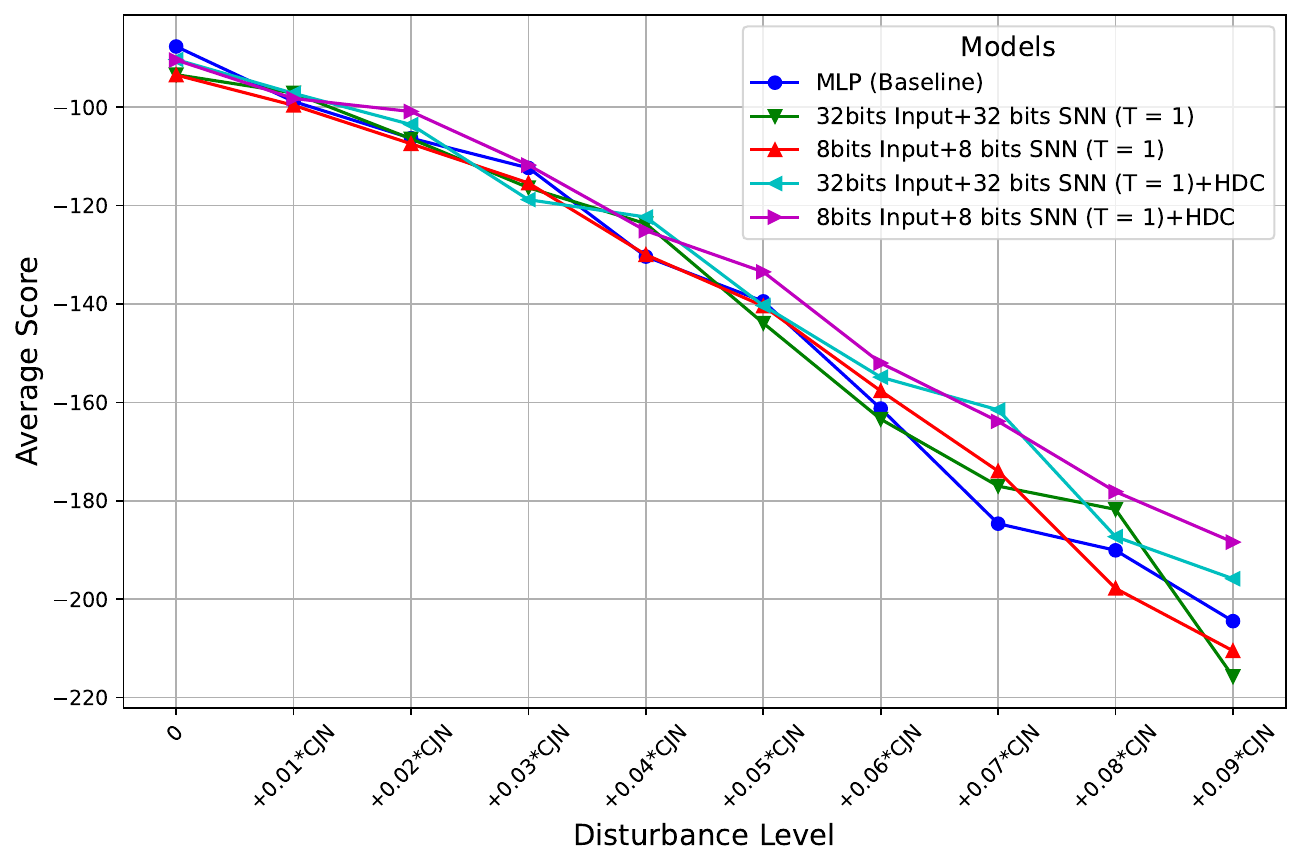}
\caption{Performance under Clock jitter noise (Acrobot)}
\label{robust_acrobot_3}
\end{minipage}

\end{figure*}
\subsubsection{Acrobot}

The Acrobot environment, inspired by Sutton's work~\cite{sutton1995generalization} and elaborated upon in Sutton and Barto's subsequent book~\cite{sutton1999reinforcement}, comprises two links connected linearly. The chain's one end is anchored, while the joint connecting the two links is actuated. The task is to apply torques to the actuated joint to raise the chain's free end above a designated height from an initial, downward-hanging position. The agent incurs a reward of -1 for each time step until the objective is achieved, thereby incentivizing quicker solutions. The episode concludes once the goal is accomplished.

In our study, we assess the efficacy of our models built on the Net2 network structure, under varying weights and input quantization degrees, incorporating hyperdimensional computing (HDC) or excluding it. For benchmarking, we employ a pure MLP model mirroring the network structure. To ensure the Acrobot's balance within a 500-step limit, each experiment is executed with 100 distinct initial settings, with the outcomes subsequently averaged for consistency.

\begin{table}[ht]
\small
\begin{threeparttable}

\begin{tabular}{c|lcc|cllc|c}
\hline
\multirow{2}{*}{Acrobot}                                                                             & \multicolumn{3}{c|}{Input Layer}                                                     & \multicolumn{4}{c|}{Intermediate/Output Layer}                                                                   & \multicolumn{1}{l}{\multirow{2}{*}{Rewards}} \\ \cline{2-8}
                                                                                                     & \multicolumn{1}{c|}{Adds} & \multicolumn{1}{c|}{Mults} & \multicolumn{1}{c|}{Energy} & \multicolumn{1}{c|}{Adds} & \multicolumn{1}{c|}{Mults} & \multicolumn{1}{c|}{Bool} & \multicolumn{1}{c|}{Energy} & \multicolumn{1}{c}{}                         \\ \hline
\begin{tabular}[c]{@{}c@{}}MLP (baseline)  (FP32)\end{tabular}                            & \multicolumn{1}{c|}{64}    & \multicolumn{1}{c|}{384}   & 1478.4pJ                    & \multicolumn{1}{c|}{3}    & \multicolumn{1}{c|}{192}   & \multicolumn{1}{c|}{0}    & 713.1pJ                     & -93.1                                        \\ \hline
\begin{tabular}[c]{@{}c@{}}32bits input+ 32bits SNN (T = 1)\\ (INT32)\end{tabular}      & \multicolumn{1}{c|}{64}    & \multicolumn{1}{c|}{384+6}   & 1219pJ                    & \multicolumn{1}{c|}{195}  & \multicolumn{1}{c|}{0}     & \multicolumn{1}{c|}{0}    & 19.5pJ                      & -92.82                                       \\ \hline
\begin{tabular}[c]{@{}c@{}}8bits input+ 8bits SNN (T = 1)\\ (INT8)\end{tabular}         & \multicolumn{1}{c|}{0}    & \multicolumn{1}{c|}{384+6}   & 100.92pJ                      & \multicolumn{1}{c|}{195}  & \multicolumn{1}{c|}{0}     & \multicolumn{1}{c|}{0}    & 5.85pJ                      & -94.21                                       \\ \hline
\begin{tabular}[c]{@{}c@{}}32bits input+ 32bits SNN (T = 1) +HDC\\ (INT32)\end{tabular} & \multicolumn{1}{c|}{0}    & \multicolumn{1}{c|}{384+6}   & 1219pJ                    & \multicolumn{1}{c|}{0}    & \multicolumn{1}{c|}{0}     & \multicolumn{1}{c|}{61}   & 0.148pJ                      & -89.17                                       \\ \hline
\begin{tabular}[c]{@{}c@{}}8bits input+ 8bits SNN (T = 1)+HDC\\ (INT8)\end{tabular}     & \multicolumn{1}{c|}{0}    & \multicolumn{1}{c|}{384+6}   & 100.92pJ                      & \multicolumn{1}{c|}{0}    & \multicolumn{1}{c|}{0}     & \multicolumn{1}{c|}{61}   & 0.148pJ                      & -92.4                                        \\ \hline
\end{tabular}
\caption{Comparison of Performance and energy consumption (Acrobot). A reward of $-x$ indicates that the model needs an average of $x$ steps to win the game.}
\label{tab:network_acro}
\end{threeparttable}
\end{table}

In HDC classification, we optimize the binary label representation, trimming its dimension from 64 to 61 by eliminating identical items across all labels. This refines energy efficiency. As depicted in Table~\ref{tab:network_acro}, our quantized models either match or outperform the baseline MLP model in terms of performance, while also achieving energy savings. For example, our model, configured with 8-bit input and weight/bias for both input and intermediate layers (SNN) and supplemented with HDC in the output layer, attains a commendable reward of -92.4 (averaging 92.4 steps to reach the target), all the while expending just 101.07 pJ of energy. This is a mere 4.61\% of the energy the MLP baseline requires.

In introducing noise, we consider the Acrobot environment's diverse six inputs. The initial four inputs range between -$\pi$ and $\pi$, and the last two lie within intervals of 4$\pi$ and 9$\pi$ respectively. We adopt a relative noise approach: $0.1*k$ noise for the first four inputs, $0.1*4*k$ for the fifth, and $0.1*9*k$ for the sixth, where 'k' spans values from 1 to 9.

Figures~\ref{robust_acrobot_0} to~\ref{robust_acrobot_3} highlight our quantized models with HDC's robustness, indicating an average 5.54\% boost across noise types, including Gaussian, Uniform, Poisson, and Clock jitter noise. Specifically, against Gaussian noise intensities of $0.5*$ and $0.7*,$ our 8-bit SNN + HDC model outperforms the MLP baseline, achieving average rewards of -133.55 and -162.68, which translate to 9\% and 10\% enhancements respectively.

The tables show that the 8-bit SNN + HDC model consistently shines at milder noise levels across all categories. However, under intense noise scenarios, the 32-bit SNN + HDC variant stands out, particularly against Gaussian, Poisson, and Clock jitter noises. Overall, these findings underscore our quantized models with HDC's advantageous performance over the traditional baseline.

\subsubsection{MountainCar}
The MountainCar markov decision process presents a deterministic challenge where a car, initially placed randomly at the bottom of a sinusoidal valley, aims to reach the right hill's peak. The car can accelerate in either direction. For every step taken, the agent receives a -1 reward until the peak is reached, incentivizing quicker solutions. The episode concludes once the car attains this goal.

Our evaluation aims to summit the peak within a 200-step constraint across 100 unique starting points, with results averaged for consistency. In this study, we utilize Network 3 (Net3), the most compact effective model identified.

Table~\ref{tab:network_mountaincar} reveals that our model, designed with 8-bit input, weight/bias for the initial and middle layers (SNN), and Hyperdimensional Computing (HDC) for the output, attains an outstanding reward of -131.43, implying an average of 131.43 steps to scale the mountain. This achievement is realized with an energy expenditure of just 35.7662 pJ, a mere 1.36\% of what the MLP baseline consumes, yet with enhanced efficacy.

\begin{figure*}[ht] 
\begin{minipage}[ht]{0.49\linewidth} 
\centering
\includegraphics[width = 0.8\linewidth]{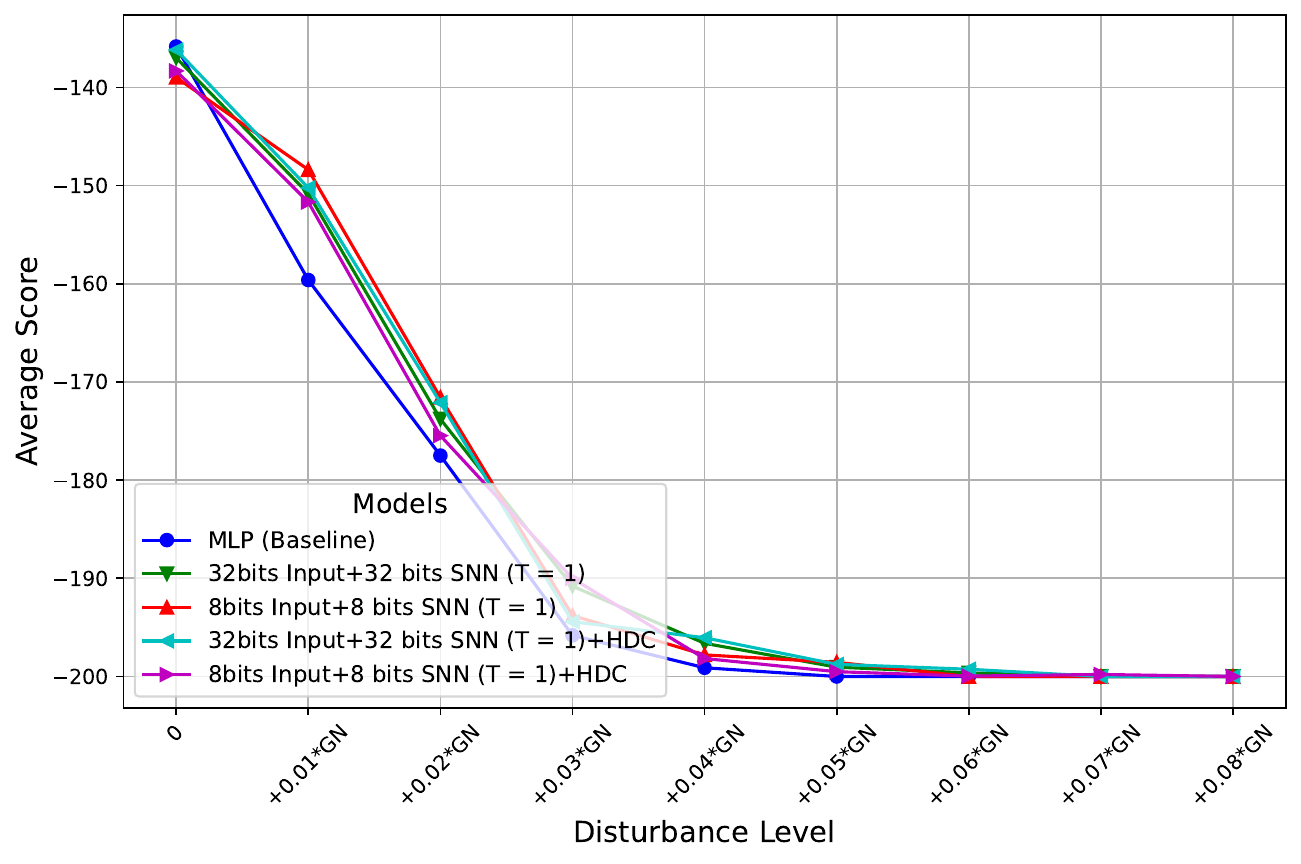}
\caption{Performance under Gaussian noise (MountainCar) } 
\label{robust_moutain_0} 
\end{minipage}
\hfill
\begin{minipage}[ht]{0.49\linewidth}
\centering
\includegraphics[width = 0.8\linewidth]{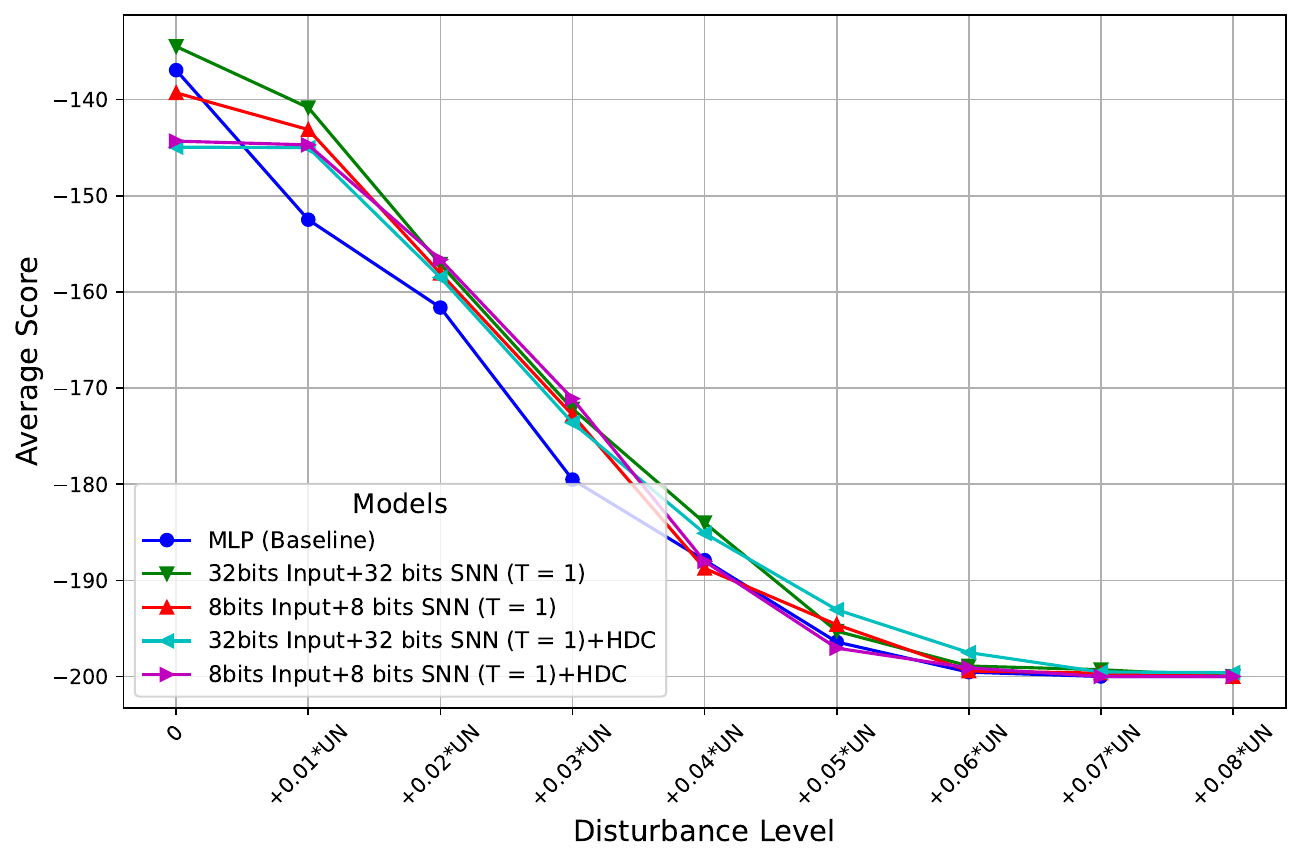}
\caption{Performance under Uniform noise (MountainCar) }
\label{robust_moutain_1}
\end{minipage}
\vspace{2ex}

\begin{minipage}[ht]{0.49\linewidth} 
\centering
\includegraphics[width = 0.8\linewidth]{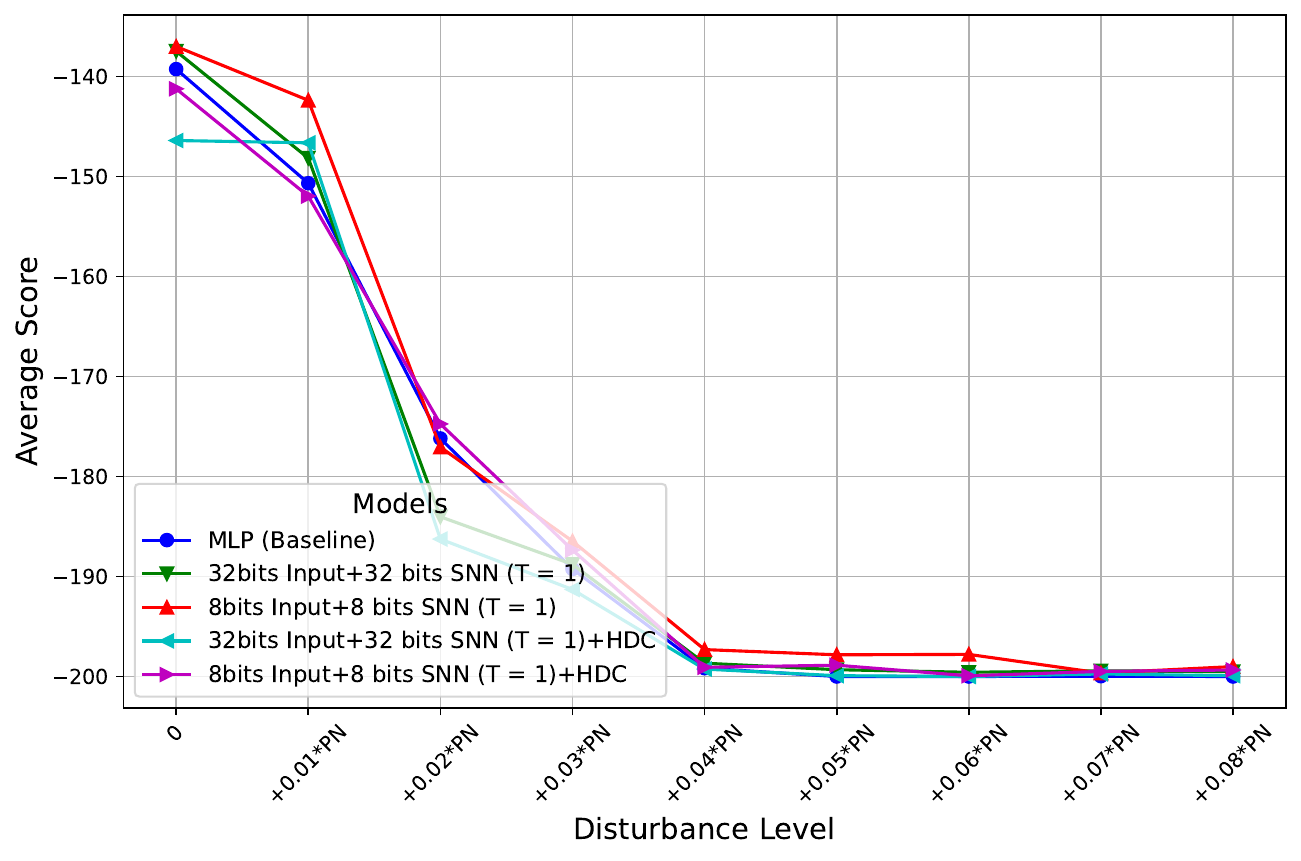}
\caption{Performance under Possion noise (MountainCar) } 
\label{robust_moutain_2} 
\end{minipage}
\hfill
\begin{minipage}[ht]{0.49\linewidth}
\centering
\includegraphics[width = 0.8\linewidth]{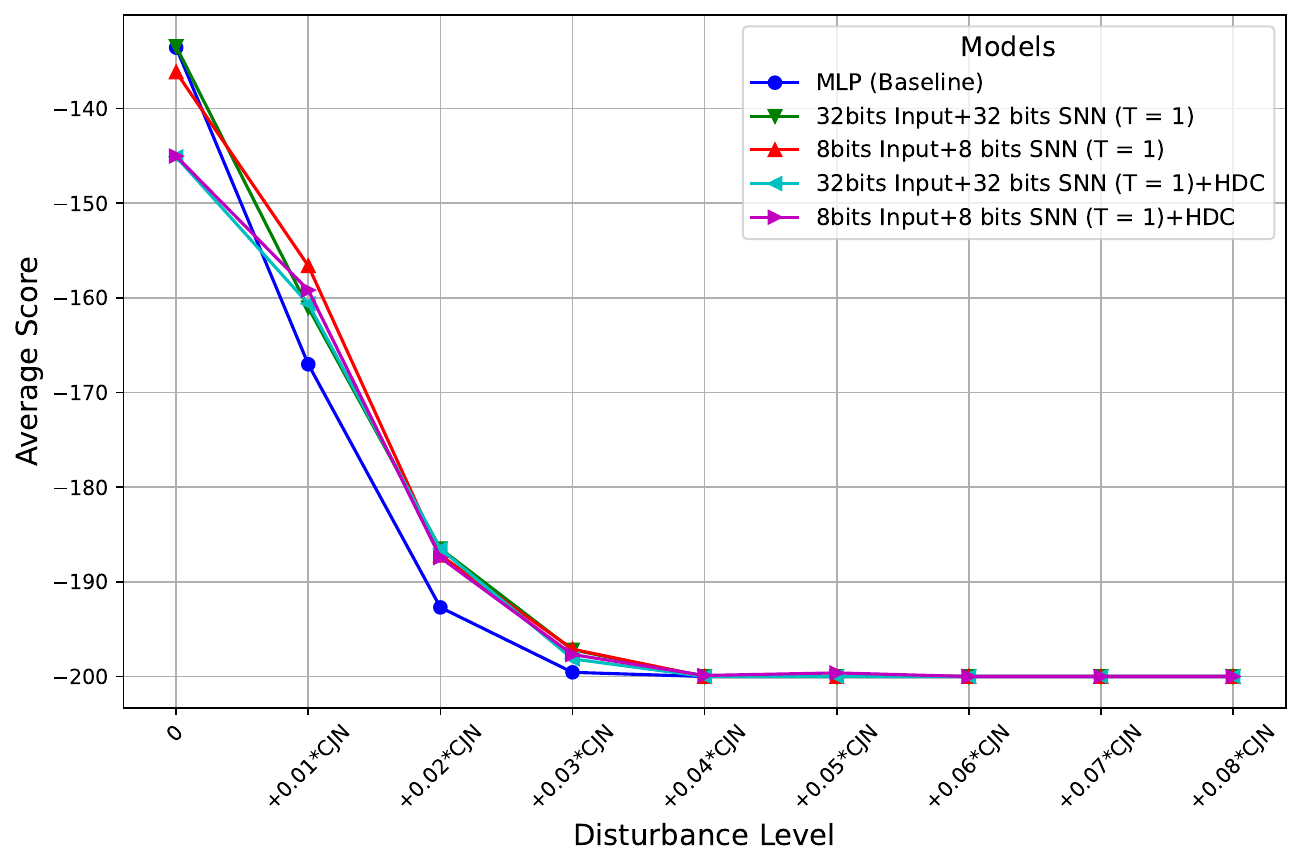}
\caption{Performance under Clock jitter noise (MountainCar) }
\label{robust_moutain_3}
\end{minipage}
\end{figure*}

As illustrated in Figures~\ref{robust_moutain_0} through ~\ref{robust_moutain_3}, our quantized models consistently outpace the MLP baseline in diverse noise scenarios. Under Gaussian Noise (GN), the MLP baseline's performance diminishes sharply, hitting the maximum penalty of -200.0 by +0.05GN. In contrast, the "32bits input+32bits SNN (T = 1)" and "32bits input+32bits SNN (T = 1)+HDC" models, among others, exhibit greater resilience as GN escalates. A similar pattern emerges with Uniform Noise (UN) where the MLP baseline lags behind. In the Poisson noise context, the MLP baseline's performance deteriorates quickly, while all SNN models sustain more modest performance declines up to +0.05PN. Regarding Clock jitter Noise (CJN), while the MLP baseline underperforms at milder noise levels, all models level to a uniform performance of -200.0 beyond +0.04CJN.

In summary, across various noise conditions, the MLP baseline consistently exhibits the most vulnerability. On the other hand, SNN models, particularly the "32bits input+32bits SNN (T = 1)" and "32bits input+32bits SNN (T = 1)+HDC", consistently demonstrate superior robustness. This underscores the potential of SNN models to deliver enhanced performance in noise-afflicted settings relative to traditional ML benchmarks.

\begin{table}[ht]
\small
\scalebox{0.9}{
\begin{threeparttable}
\begin{tabular}{c|ccc|cccc|cccc|c}
\hline
\multirow{2}{*}{MoutainCar}                                                                          & \multicolumn{3}{c|}{Input Layer}                                 & \multicolumn{4}{c|}{Intermediate Layer}                                                       & \multicolumn{4}{c|}{Output Layer}                                                            & \multirow{2}{*}{Rewards} \\ \cline{2-12}
                                                                                                     & \multicolumn{1}{c|}{Adds} & \multicolumn{1}{c|}{Mults} & Energy  & \multicolumn{1}{c|}{Adds} & \multicolumn{1}{c|}{Mults} & \multicolumn{1}{c|}{Bool} & Energy   & \multicolumn{1}{c|}{Adds} & \multicolumn{1}{c|}{Mults} & \multicolumn{1}{c|}{Bool} & Energy  &                          \\ \hline
\begin{tabular}[c]{@{}c@{}}MLP (baseline)\\  (FP32)\end{tabular}                             & \multicolumn{1}{c|}{24}    & \multicolumn{1}{c|}{48}    & 199.2pJ & \multicolumn{1}{c|}{24}    & \multicolumn{1}{c|}{576}   & \multicolumn{1}{c|}{0}    & 2152.8pJ & \multicolumn{1}{c|}{3}    & \multicolumn{1}{c|}{72}    & \multicolumn{1}{c|}{0}    & 269.1pJ & -135.59                  \\ \hline
\begin{tabular}[c]{@{}c@{}}32bits input+\\ 32bits SNN (T = 1)\\  (INT32)\end{tabular}      & \multicolumn{1}{c|}{24}    & \multicolumn{1}{c|}{48+2}    & 158.6pJ & \multicolumn{1}{c|}{600}  & \multicolumn{1}{c|}{0}     & \multicolumn{1}{c|}{0}    & 60pJ   & \multicolumn{1}{c|}{75}   & \multicolumn{1}{c|}{0}     & \multicolumn{1}{c|}{0}    & 7.5pJ   & -136.95                  \\ \hline
\begin{tabular}[c]{@{}c@{}}8bits input+\\ 8bits SNN (T = 1)\\  (INT8)\end{tabular}         & \multicolumn{1}{c|}{24}    & \multicolumn{1}{c|}{48+2}    & 17.72pJ   & \multicolumn{1}{c|}{600}  & \multicolumn{1}{c|}{0}     & \multicolumn{1}{c|}{0}    & 18pJ   & \multicolumn{1}{c|}{75}   & \multicolumn{1}{c|}{0}     & \multicolumn{1}{c|}{0}    & 2.25pJ   & -134.05                  \\ \hline
\begin{tabular}[c]{@{}c@{}}32bits input+\\ 32bits SNN (T = 1) +\\ HDC (INT32)\end{tabular} & \multicolumn{1}{c|}{24}    & \multicolumn{1}{c|}{48+2}    & 158.6pJ & \multicolumn{1}{c|}{600}  & \multicolumn{1}{c|}{0}     & \multicolumn{1}{c|}{0}   & 60pJ   & \multicolumn{1}{c|}{0}    & \multicolumn{1}{c|}{0}     & \multicolumn{1}{c|}{19}   & 0.0462pJ  & -143.1                   \\ \hline
\begin{tabular}[c]{@{}c@{}}8bits input+\\ 8bits SNN (T = 1)+\\ HDC (INT8)\end{tabular}     & \multicolumn{1}{c|}{24}    & \multicolumn{1}{c|}{48+2}    & 17.72pJ   & \multicolumn{1}{c|}{600}  & \multicolumn{1}{c|}{0}     & \multicolumn{1}{c|}{0}   & 18pJ   & \multicolumn{1}{c|}{0}    & \multicolumn{1}{c|}{0}     & \multicolumn{1}{c|}{19}   & 0.0462pJ  & -131.43                  \\ \hline
\end{tabular}

\caption{Comparison of Performance and energy consumption (MountainCar).}
\label{tab:network_mountaincar}
\begin{tablenotes}
\item {
Reward $-x$ indicates that the model needs an average step of $x$ to win the game. }
\end{tablenotes}

\end{threeparttable}}
\end{table}

\begin{figure*}[ht] 
\begin{minipage}[ht]{0.49\linewidth} 
\centering
\includegraphics[width = 0.8\linewidth]{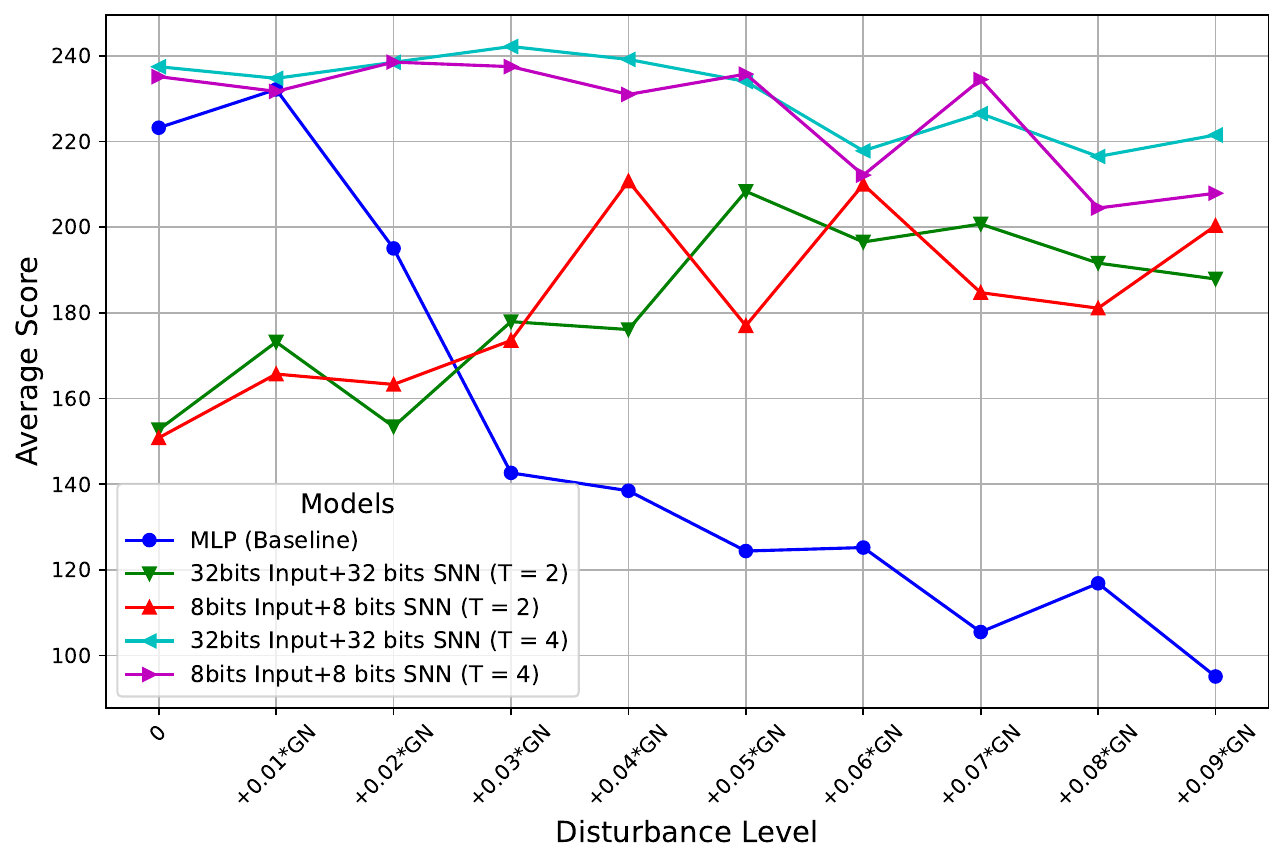}
\caption{Performance under Gaussian noise (Lunar Lander) } 
\label{robust_lunar_0} 
\end{minipage}
\begin{minipage}[ht]{0.49\linewidth}
\centering
\includegraphics[width = 0.8\linewidth]{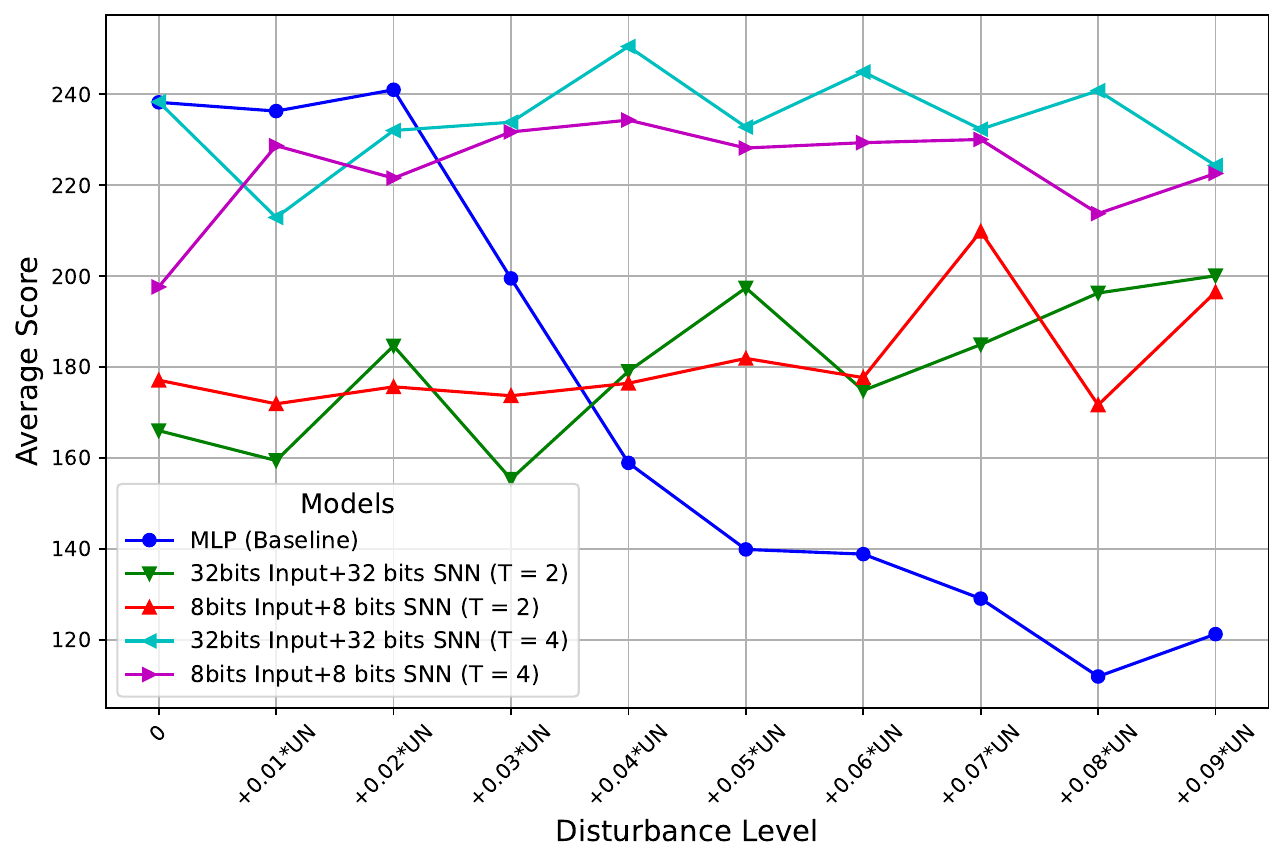}
\caption{Performance under Uniform noise (Lunar Lander) }
\label{robust_lunar_1}
\end{minipage}
\vspace{2ex}

\begin{minipage}[ht]{0.49\linewidth} 
\centering
\includegraphics[width = 0.8\linewidth]{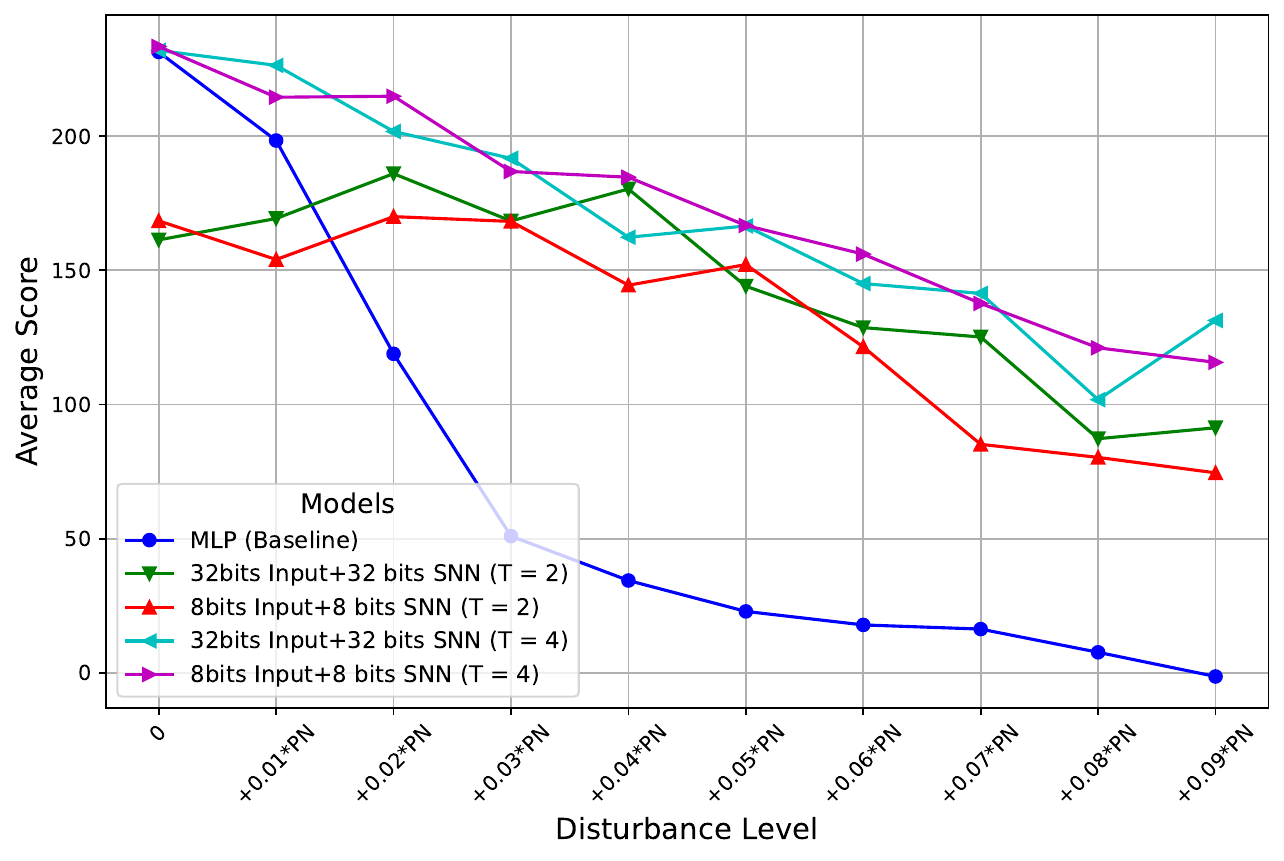}
\caption{Performance under Possion noise (Lunar Lander) } 
\label{robust_lunar_2} 
\end{minipage}
\begin{minipage}[ht]{0.49\linewidth}
\centering
\includegraphics[width = 0.8\linewidth]{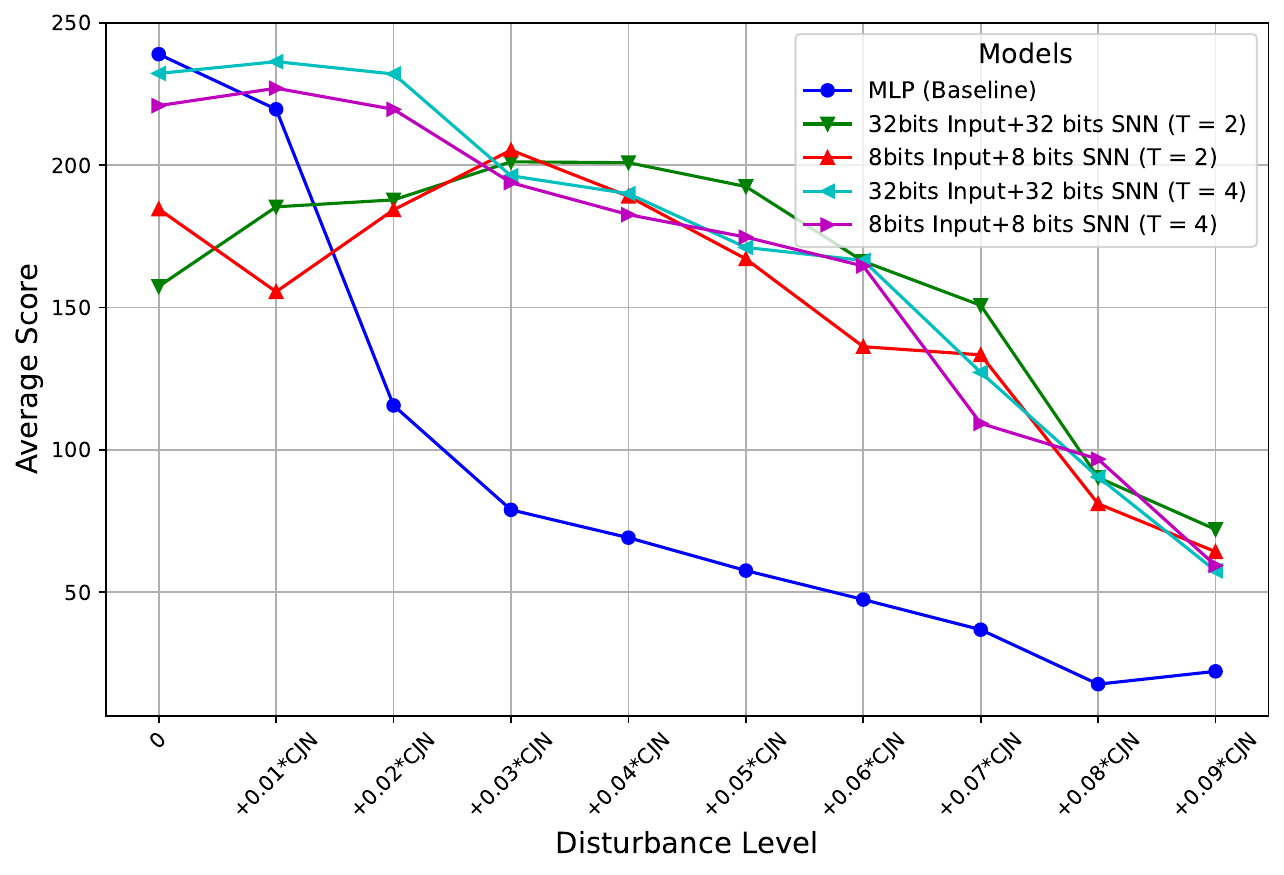}
\caption{Performance under Clock jitter noise (Lunar Lander) }
\label{robust_lunar_3}
\end{minipage}
\end{figure*}

\subsubsection{Lunar Lander}
In the LunarLander environment, an agent navigates a space lander in 2D to securely touch down between two flags, delineating the landing pad. The agent can move the lander left, right, remain stationary, or ignite the main engine. The reward structure is detailed as:

\begin{itemize}
    \item A smooth descent from the screen's top to the landing pad with stabilization earns between 100 and 140 points. Deviating from the pad leads to reward deductions.
    \item Crashing imposes a -100 point penalty, while a safe landing confers +100 points.
    \item Ground contact with each of the lander's legs grants +10 points.
    \item Using the main engine costs -0.3 points per frame, and activating the side engine reduces the reward by -0.03 points per frame.
\end{itemize}

The environment challenge is deemed resolved when the score reaches 200 points.

\begin{table}[ht]
\small
\scalebox{0.9}{
\begin{threeparttable}
\begin{tabular}{c|ccc|cccc|cccr|c}
\hline
\multirow{2}{*}{Lunar Lander}                                                                    & \multicolumn{3}{c|}{Input Layer}                                & \multicolumn{4}{c|}{Intermediate Layer}                                                      & \multicolumn{4}{c|}{Output Layer}                                                            & \multirow{2}{*}{Rewards} \\ \cline{2-12}
                                                                                                 & \multicolumn{1}{c|}{Adds} & \multicolumn{1}{c|}{Mults} & Energy & \multicolumn{1}{c|}{Adds}  & \multicolumn{1}{c|}{Mults} & \multicolumn{1}{c|}{Bool} & Energy & \multicolumn{1}{c|}{Adds} & \multicolumn{1}{c|}{Mults} & \multicolumn{1}{c|}{Bool} & Energy  &                          \\ \hline
\begin{tabular}[c]{@{}c@{}}MLP (baseline)\\  (FP32)\end{tabular}                         & \multicolumn{1}{c|}{64}    & \multicolumn{1}{c|}{512}   & 1.952J  & \multicolumn{1}{c|}{64}     & \multicolumn{1}{c|}{4096}  & \multicolumn{1}{c|}{0}    & 15.212J & \multicolumn{1}{c|}{4}    & \multicolumn{1}{c|}{256}   & \multicolumn{1}{c|}{0}    & 950.8pJ & 230.33                   \\ \hline
\begin{tabular}[c]{@{}c@{}}32bits input+\\ 32bits SNN (T = 2) \\ (INT32)\end{tabular}  & \multicolumn{1}{c|}{128}    & \multicolumn{1}{c|}{512+8}   & 1.629J  & \multicolumn{1}{c|}{8384}  & \multicolumn{1}{c|}{0}     & \multicolumn{1}{c|}{0}    & 0.8384J  & \multicolumn{1}{c|}{524}  & \multicolumn{1}{c|}{0}     & \multicolumn{1}{c|}{0}    & 52.4pJ  & 174.26                   \\ \hline
\begin{tabular}[c]{@{}c@{}}8bits input+\\ 8 bits SNN (T = 2) \\ (INT8)\end{tabular}     & \multicolumn{1}{c|}{128}    & \multicolumn{1}{c|}{512+8}   & 0.1358J  & \multicolumn{1}{c|}{8384}  & \multicolumn{1}{c|}{0}     & \multicolumn{1}{c|}{0}    & 0.256J  & \multicolumn{1}{c|}{524}  & \multicolumn{1}{c|}{0}     & \multicolumn{1}{c|}{0}    & 15.72pJ & 171.97                   \\ \hline
\begin{tabular}[c]{@{}c@{}}32bits input+\\ 32bits SNN (T = 4) \\  (INT32)\end{tabular} & \multicolumn{1}{c|}{256}    & \multicolumn{1}{c|}{512+8}   & 1.6424J  & \multicolumn{1}{c|}{16832} & \multicolumn{1}{c|}{0}     & \multicolumn{1}{c|}{0}    & 1.6832J  & \multicolumn{1}{c|}{1052} & \multicolumn{1}{c|}{0}     & \multicolumn{1}{c|}{0}    & 105.2pJ & 234.71                   \\ \hline
\begin{tabular}[c]{@{}c@{}}8bits input+\\ 8bits SNN (T = 4)\\  (INT8)\end{tabular}     & \multicolumn{1}{c|}{256}    & \multicolumn{1}{c|}{512+8}   & 0.1397J  & \multicolumn{1}{c|}{16832} & \multicolumn{1}{c|}{0}     & \multicolumn{1}{c|}{0}    & 0.5184J  & \multicolumn{1}{c|}{1052} & \multicolumn{1}{c|}{0}     & \multicolumn{1}{c|}{0}    & 31.56pJ & 227.58                   \\ \hline
\end{tabular}

\caption{Comparison of Performance and energy consumption (Lunar Lander).}
\label{tab:network_lunar}
\begin{tablenotes}
\item {
A reward of \(x\) denotes that the model accumulates a score of \(x\) before the game concludes. }
\end{tablenotes}

\end{threeparttable}
}
\end{table}

In our evaluation, each test is conducted 100 times with a cap of 1000 steps, ensuring ample chances for the lunar lander to achieve its goal. Given the environment's intricacy, the HDC and SNN models with a time step $T$ of 1 proved inadequate. To enhance the model's complexity and aptitude for discerning intricate details, we adjusted the SNN's time step to 2 and 4.

With a time step $T$ of 4 and input and weight/bias set to 8 bits, we secured a reward of 227.58. This aligns closely with the MLP baseline's 230.33, yet with a markedly lower energy usage of 0.6897 J—just 3.8\% of the baseline's 18.1148 J.

Figures~\ref{robust_lunar_0} to~\ref{robust_lunar_3} delineate the robustness contrast between the MLP baseline and our SNN configurations. Although the baseline initially shows tolerance to all noise types at zero level, it suffers a pronounced degradation with increasing noise. In contrast, specific SNN configurations, notably "32bits input+32bits SNN (T = 4)" and "8bits input+8bits SNN (T = 4)", consistently maintain robustness amidst rising noise. Remarkably, our "8bits input+8bits SNN (T = 4)" model surpasses the baseline by an average factor of 1.824 across all noise scenarios, highlighting the enhanced noise-resilient nature of SNNs compared to the MLP baseline.

\section{Related Work}

In environments constrained by computational resources, energy and memory efficiency are paramount. Various energy efficiency strategies are employed: \textbf{Early exit strategies}~\cite{bolukbasi2017adaptive,jayakodi2020design,panda2016conditional,scardapane2020should} utilize the depth of neural networks by adding multiple exit layers after convolutional layers. While it enhances adaptability, it may increase energy consumption. Notably, significant misclassifications can degrade performance versus the baseline~\cite{rashid2022template}. \textbf{Dynamic network pruning}~\cite{nikolaos2019dynamic,cai2019once,dong2019network,frankle2018lottery} iteratively trims larger networks targeting optimal performance metrics. It is tailored for networks with redundant weights/neurons. \textbf{Model compression}, incorporating techniques like binarization~\cite{courbariaux2016binarized} and quantization~\cite{hubara2021accurate,wang2019haq,fan2020training}, prioritizes model size reduction for memory efficiency. Quantization replaces FP32 operations with INT8 to reduce model size as well as conserve energy. 

In contrast to prior methodologies hinging on intricate MLPs or CNNs, which necessitate energy-intensive multiplications during inference~\cite{de2018end,lockwood2020reinforcement,wang2020novel}, our proposed method is specifically tailored for scenarios that are either acutely resource-restricted or intensely energy-conscious. We introduce HyperSNN, a framework that leverages spiking neural networks (SNNs) to replace multiplications with additions and incorporates hyperdimensional computing (HDC) to substitute additions with xor operations. Importantly, an INT8 addition is characterized by consuming a mere 0.81\% of the energy required for FP32 multiplications~\cite{horowitz20141}. The resilience and efficiency of both SNN and HDC are well-attested by~\cite{el2021securing, kundu2021hire,sharmin2019comprehensive}.

\begin{table}[ht]
\small
\scalebox{0.75}{
\begin{threeparttable}
\begin{tabular}{c|cc|cc|cc|cc}
\hline
                                        & \multicolumn{2}{c|}{Cartpole}                                                                                                                                                          & \multicolumn{2}{c|}{Acrobot}                                                                                                                                                  & \multicolumn{2}{c|}{MountainCar}                                                                                                                                      & \multicolumn{2}{c}{Lunar Lander}                                                                                                                                        \\ \hline
                                                        & \multicolumn{1}{c|}{Model structure}                                                                                 & \begin{tabular}[c]{@{}c@{}}Total Model \\ Size(KB)\end{tabular} & \multicolumn{1}{c|}{Model structure}                                                                    & \begin{tabular}[c]{@{}c@{}}Total Model \\ Size (KB)\end{tabular} & \multicolumn{1}{c|}{Model structure}                                                                    & \begin{tabular}[c]{@{}c@{}}Total Model \\ Size (KB)\end{tabular} & \multicolumn{1}{c|}{Model structure}                                                                      & \begin{tabular}[c]{@{}c@{}}Total Model \\ Size (KB)\end{tabular} \\ \hline
Chowdhury~\cite{chowdhury2021one}                                                & \multicolumn{1}{c|}{\begin{tabular}[c]{@{}c@{}}SCN(16, 5, 2), SCN(32, 5, 2),\\ SCN(32, 5,2), Fc(32, 2)\end{tabular}} & 156.04                                                          & \multicolumn{1}{c|}{/}                                                                                  & /                                                                   & \multicolumn{1}{c|}{/}                                                                                  & /                                                           & \multicolumn{1}{c|}{/}                                                                                    & /                                                           \\ \hline
\begin{tabular}[c]{@{}c@{}}Chevtchenkoa  ~\cite{chevtchenko2020learning} \end{tabular} & \multicolumn{1}{c|}{\begin{tabular}[c]{@{}c@{}}Fc(4, 100), \\ Fc(100, 2)\end{tabular}}                               & 2.74                                                            & \multicolumn{1}{c|}{\begin{tabular}[c]{@{}c@{}}Fc(6, 120), \\ Fc(120, 20),\\ Fc(20, 3)\end{tabular}}    & 12.98                                                               & \multicolumn{1}{c|}{\begin{tabular}[c]{@{}c@{}}Fc(2, 100), \\  Fc(100, 3)\end{tabular}}  & 2.35                                                      & \multicolumn{1}{c|}{/}                                                                                    & /                                                           \\ \hline
Sergio  Chevtchenko ~\cite{chevtchenko2023neuromorphic}                                     & \multicolumn{1}{c|}{/}                                                                                               & /                                                               & \multicolumn{1}{c|}{\begin{tabular}[c]{@{}c@{}}Fc(6, 60), \\ Fc (60, 300),\\ Fc(300, 500),\\Fc(500, 3) \end{tabular}} & 666.89                                                              & \multicolumn{1}{c|}{/}                                                                                  & /                                                           & \multicolumn{1}{c|}{/}                                                                                    & /                                                           \\ \hline
Wu~\cite{wu2022training}                                        & \multicolumn{1}{c|}{\begin{tabular}[c]{@{}c@{}}Fc(4, 800), \\ Fc(800, 2)\end{tabular}}                               & 21.88                                                           & \multicolumn{1}{c|}{/}                                                                                  & /                                                                & \multicolumn{1}{c|}{\begin{tabular}[c]{@{}c@{}}Fc(40, 12), \\ Fc(12, 48),\\ Fc(48, 3)\end{tabular}}        & 4.93                                                        & \multicolumn{1}{c|}{/}                                                                                    & /                                                           \\ \hline
HyperSNN                                              & \multicolumn{1}{c|}{\begin{tabular}[c]{@{}c@{}}Net1: Fc(4,10), \\ HDC(10)\end{tabular}}                              & 0.051                                                           & \multicolumn{1}{c|}{\begin{tabular}[c]{@{}c@{}}Net2: Fc(6, 64), \\ HDC(64)\end{tabular}}                & 0.461                                                               & \multicolumn{1}{c|}{\begin{tabular}[c]{@{}c@{}}Net3: Fc(2, 24), \\ Fc(24, 24),\\  HDC(24)\end{tabular}} & 0.047                                                       & \multicolumn{1}{c|}{\begin{tabular}[c]{@{}c@{}}Net4: Fc(8, 64), \\ Fc(64, 64), \\ Fc(64, 4)\end{tabular}} & 4.88                                                        \\ \hline
\end{tabular}
\caption{Comparison of SNN network structures.}
\label{tab:network_com}

\begin{tablenotes}

\item {Fc($x$, $y$) indicates the input size of spiking fully connect layer is $x$ and the output size is $y$); SCN($x$,$y$,$z$) indicates the output channel is $x$, kernel size is $y$ and stride is $z$; HDC($x$) indicates the length of hypervector is $x$}
\end{tablenotes}
\end{threeparttable}}

\end{table}

Integrating SNN with HDC offers a pioneering approach to control problems, despite the separate examination of both methodologies. Common SNN techniques often rely on expansive network architectures to enhance accuracy, deploying algorithms such as DQN, DDQN, PPO, and the Nengo model. For instance, within the Cartpole domain, \cite{chowdhury2021one} employed a 3-layer CNN alongside an additional fully-connected layer via DQN, while \cite{wu2022training} achieved a reward exceeding 500 using a two-layer SNN with 800 hidden neurons through DDQN. Yet, our NET1, with a mere 10-neuron hidden layer, matches these reward levels. In the Acrobot environment, \cite{chevtchenko2020learning} employed dual 256-neuron hidden layers with PPO, while \cite{chevtchenko2023neuromorphic} implemented an SNN strategy using layers of 120 and 20 neurons via DQN, aiming for an average reward of -100. Such architectures are substantially larger compared to our designs. Our approach in the MountainCar scenario is more streamlined than those of \cite{wu2022training} and \cite{chevtchenko2023neuromorphic}, as outlined in Table\ref{tab:network_com}. In the Lunar Lander context, the model by \cite{peters2019dynamic}, though not exhaustively detailed, employs 2000 LIF neurons for encoding and an additional 500 for classification, which is significantly larger than our 64-neuron setup. Our strategy places a premium on energy efficiency, leading to models spanning 2-3 layers and hidden sizes ranging from 10 to 64.

While few studies have explored the appropriateness of HDC for control challenges, established implementations such as QHD use a 6000-dimension, Neumann Era starts at 1024, and DARL operates at 2048 dimensions~\cite{Ni2022QHDAB,amrouch2023beyond,chen2022darl}. Optimizing these dimensions can enhance power efficiency, especially for resource-constrained devices. In our work, using the cartpole example, we streamlined the dimension to just 10, leading to a substantial reduction in energy consumption for the final classification layer, from 75pJ to 2.43fJ.

\section{Discussion}
\subsection{Model predictive control}

Model predictive control (MPC) stands apart from conventional feedback-only control methods by optimizing performance criteria over a fixed horizon~\cite{howard2010receding}. At each time increment, MPC addresses an optimization challenge based on the latest data or estimates. The primary benefits of MPC include:
\begin{itemize}
\item \textbf{Predictive nature:} It adapts to system alterations through forecasting future states.
\item \textbf{Complex system handling:} Efficiently governs multi-input, multi-output systems.
\item \textbf{Constraint management:} Ensures systems operate within defined boundaries in real-time.
\item \textbf{Real-time adjustment:} Responds to disturbances via perpetual optimization.
\end{itemize}

Nevertheless, the predominant limitation of MPC is its significant energy consumption, which can hinder widespread industrial use. Overcoming this constraint might fully actualize its capabilities in computational settings.

\begin{algorithm}[ht]
    \centering
    \caption{MPC for cartpole environment:}
    \label{alg_mpc}
    \begin{algorithmic}[1]
\REQUIRE Pretrained model $M$; State at timestep t $S_t$; Action at time step $t$ $A_t$; State Simulation by MPC $\hat{S}$; Action Simulation by MPC $\hat{A}$
\ENSURE MPC Steps $l_{mpc}$; Max time step for Cartpole $T$; Gaussion Noise $G_t$ added to each time step $t$. 

\FOR{$i=1$ to $T$}
\STATE $A_{t+1} = M(S_t+G_t)$
\STATE $\hat{A}$ = $A_{t+1}$
\FOR{$j=1$ to $l_{mpc}$}
\STATE $\hat{S}$, done = $E(\hat{A},s_{t})$
\STATE $\hat{A}$ = $M(\hat{s})$
\IF {done == True}
\STATE $A_{t+1} = 1-A_{t+1}$
\ENDIF
\ENDFOR
\ENDFOR
\end{algorithmic}
\end{algorithm}

In this work, we showcase the superior energy efficiency of HyperSNN, specifically the "8bits input + 8bits SNN + HDC" configuration. It consumes only 1.36\% to 9.96\% of the energy demanded by standard MLP approaches. Such efficiency renders the execution of MPC workflows viable, setting the stage for advanced control algorithms.

When exposed to 0.08Gaussian Noise, both HyperSNN and the conventional MLP baseline struggle in the absence of MPC. Yet, by extending MPC steps to seven — which signifies considering seven upcoming outcomes — our model effectively accomplishes the task, achieving an average reward of 780.94. Importantly, this success is achieved with only 0.8 times the energy consumption of the traditional MLP baseline, which doesn't utilize MPC. This distinction underscores the energy-saving benefits of our model, even when it is paired with sophisticated control mechanisms like MPC. The detailed results can be found in Table~\ref{tab:mpc}.

\begin{table}[ht]
\small
\begin{threeparttable}
\begin{tabular}{c|c|c|c|c|c|c|c|c}
\hline
Cartpole                                                                                       & 1 more MPC step & 2      & 3      & 4      & 5      & 6      & 7      & 8       \\ \hline
\begin{tabular}[c]{@{}c@{}}Rewards $x$ of 8bits input+ \\ 8bits SNN+ HDC with 0.08 GN\end{tabular} & 182.85          & 190.47 & 209.34 & 225.68 & 254.78 & 471.66 & 780.94 & 1282.14 \\ \hline
\end{tabular}
\caption{Ablation Study on MPC steps with 0.08GN on 8-bits model.}
\label{tab:mpc}
\begin{tablenotes}
\item {
A reward of \(x\) signifies that the cartpole maintains balance for \(x\) steps before failure.}
\end{tablenotes}

\end{threeparttable}
\end{table}

\subsection{Classification}

In evaluating the robustness of our HyperSNN models, we have also subjected them to diverse classification tasks emblematic of common wearable technology applications. This encompassed human activity recognition via the UCI-HAR dataset, speech analysis using the ISOLET dataset, and image classification leveraging the Fashion MNIST dataset. Our empirical findings reveal that by adopting 8-bit inputs alongside weights/biases, we achieve comparable performance to the 32-bit versions across all tested scenarios, provided the time window size, \( T \), is configured between 1 and 2.

Our models employ an architecture consisting of 8-bit inputs for the input layer, 8-bit weights and biases in the SNN intermediate layer, and leverage HDC for the output layer to address classification tasks. Upon evaluation against the FashionMNIST, ISOLET, and UCI-HAR datasets, our models posted accuracies of 85.49\%, 95.77\%, and 94.53\%, respectively. This performance comes, as expected, with a 69-74\% reduction in energy consumption. However, the overhead of other components (such as input normalization) of the image classification task is significantly larger, hence the energy savings are not as large compared to control tasks.

\section{Conclusion}

In this study, we explore the combined use of spiking neural networks (SNNs) and hyperdimensional computing (HDC) in a model we called HyperSNN. It is a solution to the increasing need for energy-efficient models, especially in edge computing. In particular, HyperSNN demonstrates energy efficiency, especially in applications such as wearables and smart home applications, where both computational resources and energy are constrained. The combination of using SNNs that replace costly FP32 multiplications with INT8 additions together with low dimensional HDC results in energy savings without trading off accuracy, as verified empirically by our comprehensive experiments. Furthermore, HyperSNN is also robust, showing resilience against a diverse array of noise disturbances, making it suitable for real-world applications, especially in challenging environment settings. 
We believe HyperSNN is a step towards more sustainable and efficient control systems, and the principles outlined here will inform future efforts in creating energy-efficient machine learning models for edge computing.

\bibliographystyle{ACM-Reference-Format}
\bibliography{sample-base}

\end{document}